\documentclass[10pt,letterpaper]{article}
\usepackage[margin=1.5in]{geometry}

\usepackage[utf8x]{inputenc}

\usepackage{array}
\usepackage{amsmath}
\usepackage{amsfonts}
\usepackage{amssymb}
\usepackage{amsthm}
\usepackage{graphicx}
\usepackage{multirow}
\usepackage{booktabs}
\usepackage{listings}
\usepackage{datetime}
\usepackage{hyperref}
\usepackage{csquotes}
\usepackage{adjustbox}
\usepackage[figuresright]{rotating}

\usepackage{amsmath}
\usepackage{amsfonts}
\usepackage{amssymb}
\usepackage{amsthm}
\usepackage{bm}
\usepackage{multirow}
\usepackage{subfigure}
\usepackage[ruled]{algorithm2e}
\usepackage{color}
\usepackage{stmaryrd}
\usepackage{enumitem}

\newcommand{\ex}[1]{\mathbb{E}\left[ #1 \right]}  
\newcommand{\exs}[1]{\mathbb{E}[ #1 ]}  
\newcommand{\varn}[1]{\mathbb{V}[#1]}  
\newcommand{\IR}{\mathbb{R}}  
\newcommand{\IP}[1]{\mathbb{P}\{#1\}}  
\newcommand{\HH}{\mathcal{H}}  
\newcommand{\indf}[1]{\llbracket #1 \rrbracket}
\newcommand{\eg}{\emph{e.g.}}  
\newcommand{\ie}{\emph{i.e.}}  

\newcommand{\obs}{y}

\newcommand{\ARIMA}{\text{ARIMA}}
\newcommand{\LSTM}{\text{LSTM}}
\newcommand{\EWMA}{\text{EWMA}}

\newtheorem{defn}{Definition}
\newtheorem{proposition}{Proposition}

\newtheorem{corollary}{Corollary}
\newtheorem{remark}{Remark}

\newcommand{\cCro}{\cite{croston_forecasting_1972}}
\newcommand{\cSny}{\cite{snyder_forecasting_2012}}
\newcommand{\cShen}{\cite{shenstone_stochastic_2005}}
\newcommand{\cSBA}{\cite{syntetos_bias_2001}}
\newcommand{\cTSB}{\cite{teunter2011intermittent}}

\newcommand{\tril}[2]{#1 \newline (#2)}

\usepackage{changepage}
\usepackage{authblk}
\usepackage{cite}
\usepackage{nameref,hyperref}
\usepackage[table]{xcolor}
\usepackage{microtype}
\DisableLigatures[f]{encoding = *, family = * }

\title{Intermittent Demand Forecasting with Renewal Processes}

\author[1,2]{\small Ali Caner T\"{u}rkmen\thanks{Work done before joining Amazon}\thanks{Corresponding author: {\tt atturkm@amazon.com}}}
\author[1]{\small Tim Januschowski}
\author[3]{\small Yuyang Wang}
\author[2]{\small Ali Taylan Cemgil}

\affil[1]{\footnotesize Amazon Web Services AI Labs, Berlin, Germany}
\affil[2]{\footnotesize Bo\u{g}azi\c{c}i University, Istanbul, Turkey}
\affil[3]{\footnotesize Amazon Web Services AI Labs, East Palo Alto, CA, USA}

\date{}
\begin{document}

\maketitle

\section*{Abstract}
Intermittency is a common and challenging problem in demand forecasting.
We introduce a new, unified framework for building intermittent demand forecasting models, which incorporates and allows to generalize existing methods in several directions.
Our framework is based on extensions of well-established model-based methods to discrete-time renewal processes, which can parsimoniously account for patterns such as aging, clustering and quasi-periodicity in demand arrivals.
The connection to discrete-time renewal processes allows not only for a principled extension of Croston-type models, but also for an natural inclusion of neural network based models---by replacing exponential smoothing with a recurrent neural network.
We also demonstrate that modeling continuous-time demand arrivals, i.e., with a temporal point process, is possible via a trivial extension of our framework.
This leads to more flexible modeling in scenarios where data of individual purchase orders are directly available with granular timestamps. Complementing this theoretical advancement, we demonstrate the efficacy of our framework for forecasting practice via an extensive empirical study on standard intermittent demand data sets, in which we report predictive accuracy in a variety of scenarios that compares favorably to the state of the art.



\section*{Introduction} \label{sec:introduction}


Intermittent demand forecasting (IDF) is concerned with demand data where demand appears sporadically in time \cite{croston_forecasting_1972,williams_stock_1984,eaves_forecasting_2004,boylan_classification_2008}, \ie, long runs of zero demand are observed before periods with nonzero demand.
Not only does this sparsity render most standard forecasting techniques impractical; it leads to challenges on measuring forecast accuracy \cite{hyndman_another_2006}, model selection \cite{kourentzes_intermittent_2014}, and forecast aggregation \cite{petropoulos_forecast_2015}.

Demand for large shares of inventory catalogues in manufacturing are well known to exhibit intermittency ~\cite{seeger_bayesian_2016,bose2017probabilistic}.
Intermittent demand is most likely to appear with slow-moving, high-value items that are critical to production processes.
For example, spare parts in aerospace and defense are well known to exhibit intermittent patterns \cite{syntetos_intermittent_2011,bacchetti2012spare}.
Therefore, precise estimates of forecast uncertainty, \eg, with access to forecast distributions, are vital for IDF.

IDF was recognized as a unique and challenging problem in the early 70s
\cite{croston_forecasting_1972,ward_determining_1978,schultz_forecasting_1987,dunsmuir_control_1989,willemain_forecasting_1994}. 
Upon recognizing that traditional approaches such as simple exponential smoothing (SES) led to poor forecasts in slow-moving inventory, Croston \cite{croston_forecasting_1972} proposed to independently apply exponential smoothing to consecutive positive demand sizes, and to the number of periods between each (\ie, {\em interarrival} or {\em interdemand} times).
Towards developing forecast distributions and uncertainty estimates, statistical models underlying Croston's method were proposed by  \cite{johnston_forecasting_1996, snyder_forecasting_2002,shenstone_stochastic_2005,snyder_forecasting_2012}.

In this paper, we make a sequence of observations on IDF methods proposed so far.
We start working from Croston's original insight, towards a consistent set of flexible intermittent demand models.
We explore extensions of these models in several directions, which are natural consequence of our new framework to systematically tackle IDF.
Our paper builds on early work in \cite{turkmen2019intermittent}, expanding on theoretical concepts and adding a thorough empirical analysis.
We summarize our observations and contributions below.

\begin{itemize}
	\item We draw connections between existing IDF models and renewal processes.
	We note that these two subjects of applied statistics both deal with temporal sparsity, and have been developed for planning spare parts inventories.
	\item We introduce a flexible set of discrete-time renewal process models for stationary intermittent demand. We illustrate that these models are able to capture patterns such as temporal clustering, aging, and quasi-periodicity of demand.
	\item We cast Croston-like models as instances of discrete-time conditional (self-modulating) renewal processes. Through this observation, we can introduce recurrent neural networks to recover more flexible ways in which renewal processes are modulated. 
	\item We observe that our construction applies in contexts in which individual demand events are observed in continuous time. 
	We make connections to temporal point processes and their neural variants, and apply these models in IDF. 
	\item We evaluate our approach on Car Parts, Auto, RAF, and M5 competition data sets, as well as a newly introduced data set, and show its favorable performance compared to the current state of the art.
\end{itemize}

Our paper consists of four sections.
In the first section, we start with a survey of literature on intermittent demand forecasting, with a focus on machine learning-driven methodologies.
In the subsequent section, we introduce preliminary statistical and algorithmic concepts such as renewal processes, temporal point processes, and recurrent neural networks.
In the third section, we introduce our unified framework---which we test on a wide array of synthetic and real data experiments in the last section. 

\section*{Literature Review}\label{sec:lit_review}


The literature on IDF can be traced back to the original paper by Croston \cite{croston_forecasting_1972}, and it has since grown considerably.
In the following, we provide a chronological overview of the literature on IDF, with a focus on works that highlight a model-based approach or employ neural networks. 
 
Croston~\cite{croston_forecasting_1972} observed that more accurate point forecasts of intermittent demand resulted from performing exponential smoothing separately on sequences of demand sizes and inter demand intervals.
Apart from providing a forecast heuristic, he introduced a set of models which, he posited, would lead to the forecast equations he provided. 
Later developed and corrected \cite{rao_comment_1973,syntetos_bias_2001}, Croston's method has been the most common IDF method for stockists and forecasters \cite{hyndman_forecasting:_2018}, and the {\em de facto} standard in forecasting software and libraries.
Among early extensions to Croston's method, Schultz \cite{schultz_forecasting_1987} suggested applying different smoothing constants to interarrival times and demand.
Willemain et al. \cite{willemain_forecasting_1994} verified results by \cite{croston_forecasting_1972}, who also illustrated that interdemand times exhibit autocorrelation by statistical analysis on IDF data. 

Dunsmuir and Snyder \cite{dunsmuir_control_1989} provided an explicit model for inventory control in an intermittent
demand scenario. 
Notably, they took stochastic delivery lead times of replenishment to be gamma random variables, modeling replenishment in continuous time. 
Johnston and Boylan \cite{johnston_forecasting_1996} pointed out that Croston's method did not offer an estimate of variability in demand.
They proposed separate continuous-time models, \ie, temporal point processes, for demand sizes and order arrivals. 
Notably, they assumed the order arrivals to follow a Poisson process, building on Cox's results from renewal theory and developing estimates of demand mean and variance.
They also explored different order size distributions under Poissonian order arrivals.

Syntetos \cite{syntetos_forecasting_2001} gave a survey of IDF in his thesis, along with error approximations in IDF models and some extensions for inventory control. 
Importantly, he pointed out the ``inversion bias'' in Croston's original forecast estimates and provided a revised estimator~\cite{syntetos_bias_2001}.
Syntetos later also gave a simpler approximant apart from reviewing questions around forecast accuracy measures in IDF \cite{syntetos_accuracy_2005}.
A similarly modified IDF estimate was studied by Shale et al. \cite{shale_forecasting_2006}, where the authors assumed underlying Poissonian order arrivals.
A review of these studies was given in \cite{syntetos_intermittent_2011}.

Another method for generating point forecasts was given by Teunter et al. \cite{teunter2011intermittent}. 
The authors applied exponential smoothing directly on the {\em probability} of demand, bypassing the inversion bias in Croston's method. 
A review, as well as a comparative study of point forecast methods on real IDF data was presented by Babai et al. \cite{babai2014intermittent}, with no strong evidence in favor of any of the forecast functions considered.

The problem of associating a stochastic model to Croston's method was explored by Snyder \cite{snyder_forecasting_2002}. 
He proposed several modifications to the method and considered intermittent demand forecasts via the parametric bootstrap on a set of state-space models (SSM).
Notably, this work draws connections between the renewal-type model of Croston \cite[Appx. B]{croston_forecasting_1972} and single source of error SSMs. 

Shenstone and Hyndman \cite{shenstone_stochastic_2005} investigated the validity of Croston's methods on several fronts, pointing out that the i.i.d. assumption on order sizes and interarrival times were inconsistent with the exponentially weighted moving average (EWMA) forecast estimates. 
The authors explored a set of ``modified'' models that would yield Croston's estimates, assuming both sizes and intervals follow an ARIMA(0,1,1) process. 
To ensure nonnegativity, required of both sizes and intervals, they proposed log-Croston models where an ARIMA(0,1,1) process determined the required quantities in the log domain. 
Importantly, the authors argue that there exist no practicable models that yield Croston's estimates as optimal forecasts. 
This is since any local-level model that yields the EWMA as an unbiased one-step-ahead forecast, and is defined on the positive real half-line, suffers from a convergence problem \cite{grunwald_properties_1997,akram2009exponential}, \ie, draws from the model converge to 0 over long terms.
Let us note that the question of a stochastic model for IDF was also raised in \cite{dolgui_bayesian_2005, sandmann_stochastic_2009}.
An extensive discussion on model-based IDF is given in \cite[Ch. 16]{hyndman_forecasting_2008}.

Snyder et al. \cite{snyder_forecasting_2012} proposed an extended set of models, contrasting several ways in which stochastic Croston-like models were considered in the literature. 
They compared Poisson, negative binomial and zero-inflated Poisson distributions for demand sizes. 
For all distributions, they tested ``static'' (constant) parameters as well as undamped (IMA) and damped (AR) dynamic parameterizations.
The proposed class of ``dynamic'' models are simple, and are shown to outperform Croston's method and static parameter models.
In this paper, we connect previous model-based approaches in \cite{shenstone_stochastic_2005} and \cite{snyder_forecasting_2012} to the rich theory behind renewal processes. 
As such, for the first time, we can derive Croston's methods and other proposed models from a principled probabilistic framework.

SSM approaches were considered and compared by Yelland \cite{yelland_bayesian_2009}, who also later proposed a hierarchical Bayesian treatment of the problem \cite{yelland_bayesian_2010}.
Seeger et al. \cite{seeger_bayesian_2016} considered Bayesian SSM with multiple sources of error, focusing on applicability on very large inventory catalogues (see also 
\cite{seeger2017approximate} for more details).
Their method incorporates exogenous features, and presents an efficient approximate inference algorithm enabling large-scale and distributed learning.
Altay et al. \cite{altay_adapting_2008} explored an adaptation of Croston's method with Holt-like trends. 
Seasonality in IDF, with a comparison of Holt-Winters and SARIMA, was considered in \cite{gamberini2010forecasting}.
Recently, a single source of error SSM with multiplicative errors was explored by Svetunkov and Boylan \cite{svetunkov2017multiplicative}.
A novel method for forecast aggregation via temporal hierarchies, THieF, was introduced by Kourentzes and Athanasopoulos \cite{kourentzes2019elucidate}.

Using machine learning methods---mainly, neural networks---in forecasting is both a recently growing area of research and the subject of ongoing debate \cite{chatfield1993neural,crone2011advances}.
In their extensive comparison of machine learning methods to traditional ``statistical'' methods, Makridakis et al. \cite{makridakis2018statistical} find scarce evidence in favor of using neural networks or other ML-based methods in forecasting.
The authors argue that most empirical evidence in forecasting favors ``model-driven'' methods over ``data-driven,'' in the terms of Januschowski et al. \cite{januschowski2019criteria}.
Moreover, data-driven methods are often harder to train and replicate; and require substantially greater effort for implementation and computation. The findings from this study are controversial to say 
the least as practical evidence from industrial researcher has consistently pointed to contrary conclusions (see e.g.,~\cite{salinas2017deepar,laptev2017,mqcnn,kasun19}). Supporting this point of view are the results from highly visible M4 competition \cite{makridakis2018m4} which was won by a hybrid method that combined exponential smoothing, recurrent neural networks, and ensembling \cite{smyl2020hybrid}. 
Other such hybrid approaches have appeared recently, mainly in the machine learning literature \cite{kourentzes2014neural,salinas2017deepar,wang2019deep,rangapuram2018deep,gasthaus2019probabilistic}, along with software libraries for neural network time series forecasting \cite{alexandrov2019gluonts}.
A thorough review on the use of neural networks in forecasting was recently given by Benidis et al. \cite{benidis2020neural}.

Several studies have considered neural networks in the context of IDF. 
Gutierrez et al. \cite{gutierrez_lumpy_2008} experimented with a single hidden layer feedforward network, with only three hidden units, trained to predict the demand on the next time interval, given time since last demand and the last observed demand size. 
They reported favorable results compared to Croston and Syntetos-Boylan methods. 
These results were debated by Kourentzes \cite{kourentzes_intermittent_2013}, who compared several neural network architectures to a suite of naive and Croston-based forecasters and demonstrated low forecast accuracy in terms of mean absolute error (MAE).
However, the author also found favorable performance by neural networks when inventory control metrics were considered directly.
Mukhopadhyay et al. \cite{mukhopadhyay2012accuracy} experimented with several training strategies to find that neural networks outperform traditional methods.
Recurrent neural networks were considered in \cite{pour2008hybrid}, and several extensions were explored in \cite{lolli2017single}.
A deep LSTM, in the context of IDF, appeared recently in \cite{fu2018hybrid}.
Our paper allows combining neural networks with a model-based approach, a first in the literature for intermittent demand forecasting. 
As a consequence, the neural forecasting methods we present here are the first dedicated neural network-based models for IDF which yield probabilistic forecasts.

Finally, we remark that IDF requires special attention when evaluating forecasting accuracy as standard forecast accuracy metrics are well known to fail in this context. 
Hyndman \cite{hyndman_another_2006} proposed MASE, mean absolute error scaled against a naive-one forecast, for the IDF task.
Kim and Kim \cite{kim_new_2016} advanced this notion to a reportedly more robust metric, employing trigonometric functions for comparing forecasts and ground truth.
Although we used both of these metrics in our evaluations, we found both approaches to be potentially misleading, as a ``zero forecast'' often yielded better results in terms of MASE or MAAPE than any other method.
Therefore, in our experiments we report mean absolute percentage error (MAPE), its symmetric version (sMAPE), and root mean squared error (RMSE). 
Moreover, as our methods are primarily geared towards probabilistic forecasting, we use P50 and P90 loss metrics as in \cite{salinas2017deepar}.
Finally, we also report root mean squared scaled error (RMSSE), the main error metric used in the recent M5 competition \cite{m5guide}. 
The question of how to define intermittency, \ie, which time series can be classified as intermittent, was studied in  \cite{williams_stock_1984,eaves_forecasting_2004,syntetos_categorization_2005,kostenko_note_2006,boylan_forecasting_2008}.

\section*{Preliminaries}\label{sec:preliminaries}

\subsection*{Problem Setup}

We consider univariate, nonnegative, integer-valued time series, corresponding to random variables denoted $\{Y_n\}_{n \in \{1, 2, \cdots\}}, Y_n \in \mathbb{N}$.
Here, $n$ indexes uniformly spaced intervals in time, each corresponding to a {\em demand review period}.
Realizations $\{y_n\}$ of $\{Y_n\}$ will typically contain long runs of zeros, \ie, only rarely will $y_n$ be greater than zero.

For a clearer exposition of our models, we will mostly use the {\em size-interval} notation for intermittent demand. 
As many $Y_n$ are zero, it will be useful to define an auxiliary index $i \in \mathbb{N}$, such that $i$ indexes the {\em issue points}, \ie, periods where nonzero demand is observed. 
More formally, we define the one to one map $\sigma(i) = \min \{n \mid \sum_{m=1}^n \indf{Y_m > 0} \ge i\}$, where $\indf{.}$ is the indicator function (Iverson bracket).
We denote {\em interdemand} times---number of periods between issue points---$Q_i$. 
That is, $Q_i = \sigma(i) - \sigma(i-1)$, taking $\sigma(0) = 0$.
We denote demand sizes for issue points $M_i = Y_{\sigma(i)}$.
Random variables $M_i, Q_i$, both defined on positive integers, fully determine $Y_n$ and vice versa.

%

We use corresponding lowercase Latin letters to denote instantiations of random variables denoted in uppercase Latin.
$\ex{\cdot}$ and $\varn{\cdot}$ denote mathematical expectation and variance respectively.
Unless otherwise specified, we keep indices $i, n$ analogous to their function in this introduction, indexing issue periods and review periods respectively.
Range indexes will denote consecutive subsequences of a random process or realization, \eg, $Y_{1:k} = \{Y_1, Y_2, \cdots, Y_k \}$.
We reserve $\mathcal{N}, \mathcal{G}, \mathcal{PO}, \mathcal{NB}, \mathcal{E}$ to denote Gaussian, geometric, Poisson, negative binomial, and exponential distributions respectively.  
Unless otherwise noted, these distributions are parameterized in terms of their mean.
The support of Poisson, geometric, and negative binomial distributions are {\em shifted}, and defined only on positive integers.
This less common form of the negative binomial distribution is detailed in the appendix.

Our main aim here is to characterize, or approximate, forecast distributions. 
Typically, these are conditional distributions of the form
$\IP{Y_{n+1} | Y_{1:n} = y_{1:n}},$ or $\IP{Y_{n+1:n+k} | Y_{1:n} = y_{1:n}}.$
We will let $\hat{Y}_{n}$ refer to an estimator available at time $n$.
Often, it will be an estimator of the one-step-ahead conditional mean $\ex{Y_{n+1} | Y_{1:n}=y_{1:n}}$.

\subsection*{Croston's Method}\label{subsec:croston}

Here, we briefly elaborate on Croston's method and some of its variants, setting up notation for discussions to follow.

In his paper~\cite{croston_forecasting_1972}, Croston highlighted an important drawback of SES in IDF. 
SES forecasts, by placing the highest weight on the most recent observation, would lead to the highest forecasts just after a demand observation, and the lowest just before. 
Heuristically, Croston proposed to separately run SES on interdemand times and positive demand sizes.  
Concretely, he set,
\begin{subequations}\label{eq:croston_fcast}
\begin{align}
\hat{M}_{i+1} &= \alpha M_{i+1} + (1 - \alpha) \hat{M}_{i}, \label{eq:croston_fcast_1}   \\
\hat{Q}_{i+1} &= \alpha Q_{i+1} + (1 - \alpha) \hat{Q}_{i}. 
\end{align}
\end{subequations}
That is, he proposed to use the EWMA of each series as the forecast estimate.
We will denote this recursive computation $\hat{Q}_{i+1} = \EWMA_\alpha (Q_{1:i+1})$.
Moreover, Croston also discussed possible models for intermittent demands, setting
\begin{subequations}\label{eq:croston_72_model}
\begin{align}
    M_i &\sim \mathcal{N}(\mu, \sigma^2) \text{  i.i.d.}, \\
    Q_i &\sim \mathcal{G}(1 / \pi) \text{  i.i.d.}
\end{align}
\end{subequations}
The paper also considers $\{M_i\} \sim \ARIMA(0, 1, 1)$ with Gaussian innovations. 

The discrepancy between the model of (\ref{eq:croston_72_model}) and forecasts (\ref{eq:croston_fcast}) has been a common issue in the IDF literature that followed.
Assuming the model in (\ref{eq:croston_72_model}), unbiased forecasts of $M_{i+1}, Q_{i+1}$ are possible simply by setting them to the averages of previous values. 
Indeed, this is ``implied'' by the i.i.d. assumption.
Instead, however, the forecast heuristic (\ref{eq:croston_fcast}) has been praised for its ability to capture serial correlation and nonstationary behavior in both sizes and intervals. 
This is clearly at odds with the model assumption.


More formally, the EWMA only functions as an {\em asymptotically} unbiased estimator of the parameter $\mu$. 
Moreover, as noted by Syntetos and Boylan in a series of papers \cite{syntetos_forecasting_2001,syntetos_bias_2001,syntetos_accuracy_2005}, the EWMA of previous interdemand times only results in a biased estimate of $\pi$, due to an oversight of the inversion bias $\ex{1 / \hat{Q}_{i}} \neq 1 / \ex{\hat{Q}_{i}}$.
In \cite{syntetos_accuracy_2005}, the authors corrected for this bias via a Taylor series approximation about the mean of $\ex{1 / \hat{Q}}$ and gave an approximation to an asymptotically unbiased forecast,
\[
\hat{Y}_n = \left(1 - \alpha/2\right) \frac{\hat{M}_{i'}}{\hat{Q}_{i'}}.
\]
Several other variants of forecasts (\ref{eq:croston_fcast}) have been explored, \eg, using a simple moving average instead of the EWMA \cite{shale_forecasting_2006, boylan_forecasting_2008}.

Both models suggested by Croston allow trajectories with negative and non-integer values.
To alleviate this misspecification issue, models with non-Gaussian likelihood have been proposed, \eg, by parameterizing a distribution of positive support with a mean process obeying an IMA model.
Nevertheless, such models---namely, \emph{nonnegative EWMA} models---are ``ill-fated'' \cite{grunwald_properties_1997} since their trajectories converge to 0 over longer terms (see the appendix).
Shenstone and Hyndman \cite{shenstone_stochastic_2005} use this result to point out that no models exist that yield Croston's forecast estimates as consistent and unbiased estimators while being immune to the convergence problem outlined here. 

Exploring possible ways to account for the model--forecast discrepancy in Croston's method,  \cite{snyder_forecasting_2002,shenstone_stochastic_2005,hyndman_forecasting_2008,snyder_forecasting_2012} studied numerous alternative probabilistic models. 
In this paper, we will consider four variations of these models as baseline methods, which we summarize in Table~\ref{tab:baselines}.
Let us note that, to the best of our knowledge, only the first three models were previously proposed \cite{hyndman_forecasting_2008,snyder_forecasting_2012}.

\begin{table}[!ht]
	\caption{Baseline models}
	\label{tab:baselines}
	\centering
	\begin{tabular}{p{4cm} p{4cm} p{4cm}}
		\toprule
		{\bf Model} 	 & $Q_i$ & $M_i$ \\ \midrule
		Static G-Po \cite{croston_forecasting_1972} 	& $\mathcal{G}(\mu_q)$, i.i.d. & $\mathcal{PO}(\mu_m)$, i.i.d. \\[10pt]
		Static G-NB \cite{snyder_forecasting_2012}    & $\mathcal{G}(\mu_q)$, i.i.d. & $\mathcal{NB}(\mu_m, \nu_m)$, i.i.d.\\[10pt]
		\midrule
		EWMA~G-Po \cite{shenstone_stochastic_2005} & {$\!
			\begin{aligned}
			&\mathcal{G}(\hat{Q}_{i-1})\\
			&\hat{Q}_{i-1} = \EWMA_\alpha(Q_{1:i-1})
			\end{aligned}$
		} &
		{$\!
			\begin{aligned}
			&\mathcal{PO}(\hat{M}_{i-1})\\
			&\hat{M}_{i-1} = \EWMA_\alpha(M_{1:i-1})
			\end{aligned}$
		}  \\[25pt]
		EWMA G-NB 	 & {$\!
			\begin{aligned}
			&\mathcal{G}(\hat{Q}_{i-1})\\
			&\hat{Q}_{i-1} = \EWMA_\alpha(Q_{1:i-1})
			\end{aligned}$
		} &
		{$\!
			\begin{aligned}
			&\mathcal{NB}(\hat{M}_{i-1}, \nu_m)\\
			&\hat{M}_{i-1} = \EWMA_\alpha(M_{1:i-1})
			\end{aligned}$
		}  \\[25pt]  \midrule
	\end{tabular}
\end{table}


%

\subsection*{Renewal Processes in Discrete Time}
\label{subsec:prelim_renewal}

Renewal processes constitute a central theme in the theory of stochastic processes and their study plays a more general role in probability theory \cite{cox_renewal_1962,cinlar2013introduction}.
Broadly, renewal processes are concerned with recurring events in repeated trials, where after each such event, trials start identically from scratch. 
That is, {\em interarrival} times between events are independent and identically distributed.
For example, in independent spins of a roulette wheel, drawing a certain number is a recurrent event as is observing three consecutive reds. 

The example of a roulette wheel may appear unusual, as spins constitute a discrete sequence of outcomes. 
Renewal processes are mostly introduced as {\em continuous-time} counting processes (temporal point processes), where inter-arrival times are i.i.d. \cite{cox_renewal_1962}. 
Yet, some of the earlier treatments of renewal theory consider the case of recurrent events in discrete time \cite[Ch. 13]{feller_introduction_1957}. 
We introduce such processes, focusing our attention to the special case where these recurrent events are simple binary outcomes.
\begin{defn}\label{def:dtrp} (Discrete-time renewal process with Bernoulli trials)
Let $\{Z_n\}_n$ define a sequence of (not necessarily independent) binary random variables.
Let $i$ index time steps where $Z_n = 1$, as above, defined through the map $\sigma(i) = \min\{n | \sum_{m=1}^n Z_m \ge i\}$, $\sigma(0) = 0$, and $Q_i = \sigma(i) - \sigma(i-1)$.
The sequence $\{Z_n\}$ defines a \emph{discrete-time renewal process (DTRP)} if $Q_i$ are i.i.d.
\end{defn}
Our reuse of notation is not coincidental. 
It should be clear that, 
\begin{remark} \label{rem:croston_is_renewal}
The demand arrival process $\{Y_n > 0\}$ of {\em Static G-Po}, and {\em Static G-NB} models, as well as Croston's original models, are DTRPs.
\end{remark}
In fact, these models rely on the most basic renewal process---the Bernoulli process---as a demand arrival regime. 

Renewal processes, as the name implies, were developed in the context of modeling {\em failures} in large systems---towards determining when parts of such systems would have to be {\em renewed}.
The underlying intuition is that every time an identical part is {\em renewed}, the random process with which it fails starts anew.

Thinking in terms of renewal processes paves the way to introducing a class of interpretable, parsimonious and flexible models.
It also gives access to tools from renewal theory, such as characterizing multi-step ahead forecast distributions with convolutions of interdemand times, or model fitting via moment-matching methods.
As we will see, they will also enable an intuitive connection to temporal point processes, introduced below.

\subsection*{Recurrent Neural Networks}

Recurrent neural networks (RNN) rely on a recurring {\em hidden state} to model sequential data, and have been widely adopted in time series applications. 
In contrast to feedforward neural networks, which can be seen as general function approximators, RNNs approximate sequences that (partly) depend on a recurrence relation.
See \cite[Ch. 10]{goodfellow2016deep} for an introduction to RNNs.
Long short term memory (LSTM) \cite{hochreiter1997long} networks, which we use in this paper, define a specific RNN architecture designed to capture long-term interactions.

Concretely, given inputs $\{\mathbf{x}_i\}$, LSTMs---as in all RNNs---obey the recurrence
\[
\mathbf{h}_i = \LSTM (\mathbf{h}_{i-1}, \mathbf{x}_i)
\]
where $\mathbf{h}_i \in \IR^d$ comprises both the memory (state) and the output of an LSTM cell.
Note that in time series applications, inputs often include past time series values.
For example, if $y_i$ denotes the time series of interest, $\mathbf{x}_i = [y_{i-1}, y_{i-2}, \cdots]$.
The exact functional form of $\LSTM(.)$ is composed of several nonlinear projections and compositions of $\mathbf{h}_{i-1}$ and $\mathbf{x}_i$, that provide mechanisms for representing ``long-term memory.''
The LSTM network is implemented on virtually all deep learning frameworks such as \cite{chen2015mxnet}.
For further details on the exact computation and optimization of LSTMs, we refer the reader to the tutorial by Olah \cite{olah2015understanding}.

Most previous works in forecasting rely on RNNs to approximate a time series in squared or absolute error.
More precisely, the training objective is
\begin{align*}
\min_\Theta\, &||\mathbf{y} - \mathbf{\hat{y}}||_2 \\
& \text{s.t.  } \hat{y}_i = g(\mathbf{h}_i) \qquad 
\mathbf{h}_i = \LSTM_\Theta(\mathbf{h}_{i-1}, \mathbf{x}_i),
\end{align*}
where $\Theta$ denotes the set of all parameters of function $\LSTM$, and $g(.)$ is a suitable projection to the domain of $y_i$, such as an affine transformation from $\IR^d$ to $\IR$.
This is the case with previous work on IDF \cite{pour2008hybrid,lolli2017single}.
A more recent approach in neural forecasting considers RNNs not as forecast functions, but to parameterize forecast {\em distributions} \cite{salinas2017deepar}, or approximate transition and emission dynamics in an SSM \cite{rangapuram2018deep}.
In this case of {\em probabilistic} RNNs \cite{januschowski2020criteria}, one assumes
\begin{align*}
y_i \sim \mathcal{D}(\beta)  \qquad \beta = g_\beta(\mathbf{h}_i),
\end{align*}
where $y_i$ is {\em drawn} from a suitable probability distribution denoted by $\mathcal{D}$ with parameters $\beta$. 
The parameters are then computed from the output of the LSTM, via an appropriate projection to their domains, denoted here by $g_\beta$.
Then, instead of the mean squared error objective, one minimizes the negative log likelihood of the model.
In addition to naturally reflecting domain assumptions about $y_i$, such as its support or tailedness, this approach naturally leads to forecast distributions.
In the sequel, when we use neural networks in the IDF context, we will rely on this approach. 
Namely, we will let neural networks parameterize conditional interarrival time distributions in self-modulating renewal processes.

\subsection*{Temporal Point Processes} \label{subsec:prelim_tpp}

Temporal point processes (TPP) are probabilistic models for sets of points on the real line.
Often, these points represent as discrete (instantaneous) {\em events} in continuous time.
Examples of discrete event sets are plenty, such as order arrivals in a financial market, anomalous log entries in a computer system, or earthquakes.

Demand arrivals in a stock control system can be seen as discrete events, and modeled via a TPP.
The rationale for this approach is more intuitive for intermittent demands where instances are rare and the main goal is to determine {\em when} they will occur. 
Croston \cite{croston_forecasting_1972} observed in his original article that intermittent demand data arose since ``updating [occurred] at fixed unit time intervals, which are much shorter than the times between successive demands for the product.''
The same approach was later used as a basis in deriving forecast approximations, for example,
\cite{johnston_forecasting_1996} based their analysis on an underlying Poisson process.

As noted above, DTRPs have better known, continuous time counterparts which constitute a large subclass of TPPs \cite[Ch. 4]{daleyintroduction}.
As such, discrete-time IDF models we introduced in Section~\ref{subsec:prelim_renewal} have natural continuous time counterparts, including those that benefit from recurrent neural networks for modulating interdemand time distributions \cite{du_recurrent_2016,mei_neural_2017,turkmen2019fastpoint,shchur2019intensity}.

Then again, discrete-time processes with binary trajectories, such as DTRPs as given in Definition~\ref{def:dtrp}, have also been referred to as ``point'' processes.
Determinantal point processes constitute an important example, which define random processes for finite-dimensional binary vectors \cite{borodindeterminantal,kulesza2012determinantal}.
The ``points'' in this case are the vector elements which equal to 1.
Although point processes customarily refer to random point collections on a Borel set, this flexible use of nomenclature further justifies the connection between DTRPs and continuous-time models.

Modeling demand directly with a point process is mostly impractical as exact timestamps of demand events are not observed. 
Moreover, estimation and inference algorithms are computationally intensive: computation times scale asymptotically in the number of events and not the number of demand review periods.
However, in increasingly many scenarios such as in e-commerce, the exact demand time is known and can be used directly. 
In IDF, the number of events is often on the same order as the number of intervals, making computational costs comparable.
In this light, we will explore using TPPs on intermittent demand data, and study how directly addressing demand timestamps affects forecast accuracy.

\section*{Models} \label{sec:models}
 
\subsection*{Discrete-Time Renewal Processes} \label{subsec:models_dtrp}

In previous sections, we discussed how existing IDF models could be cast as renewal processes.
We will first follow the obvious extension implied by this generalization, and we will introduce a set of models with more flexible interdemand time distributions.
Particularly, we will first consider extending the demand arrival process of {\em Static} models given above. 
In this section, we focus on the rationale for doing so.

Previous IDF models have considered the geometric distribution as a model for the times between successive issue points.
Doing so is well-justified, it assumes a homogeneous Bernoulli process for demand arrivals, with the probability of observing a demand constant across time and independent of all other intervals.
Yet, this is an obvious oversimplification of real-world arrival processes. 

A natural generalization of {\em Static} models is to alter the demand interarrival time distribution.  
This relaxes the strict assumption imposed by homogeneous Bernoulli arrivals.
To frame our analysis, we define the {\em hazard rate}\footnotemark\, $h(k)$ for a discrete random variable $Q$ with c.d.f. $F$ as
\[
h(k) = \IP{Q = k | Q \ge k} = \IP{Q = k}/\IP{Q \ge k} = \dfrac{\IP{Q = k}}{1 - F(k-1)}.
\] 
Letting $Q$ denote interarrival times, the interpretation of the hazard rate is intuitive. 
It stands for the probability of a new demand occurring at time $k$, given that it hasn't occured for $k-1$ intervals since the previous issue point. 

\footnotetext{Our terminology follows the {\em hazard function} in continuous-time renewal processes.}

It is not hard to show that when $Q$ admits a geometric distribution, the hazard rate is constant and consequently {\em independent} of the past. 
This property of the geometric distribution has been referred to as the {\em memoryless} property.

If $h(k)$ is monotonically increasing, the function determines a profile for {\em aging}.
This is often the case in machine parts, where the probability of a part being renewed increases with the time in service. 
$h(k)$ can also be decreasing, or imply {\em negative aging}. 
In such demand data, issue points are typically {\em clustered} in time.
A flexible distribution for $Q$ can also be used to capture quasi-periodic demand, \eg, with a distribution for $Q$ that concentrates around a set period $\mu$. 

We replace the interarrival time distribution in {\em Static} models with the negative binomial distribution.
The negative binomial can flexibly capture aging, clustering, and periodic effects discussed above, as illustrated in Fig~\ref{fig:neg_bin_dists}.
Indeed, the geometric distribution is a special case of the negative binomial, confirming our intuition that it determines, in a sense, the most basic renewal process possible.
We give further details on the negative binomial distribution in the appendix.

\begin{figure}[!h]
	\makebox[\textwidth][c]{\includegraphics[width=1.3\textwidth]{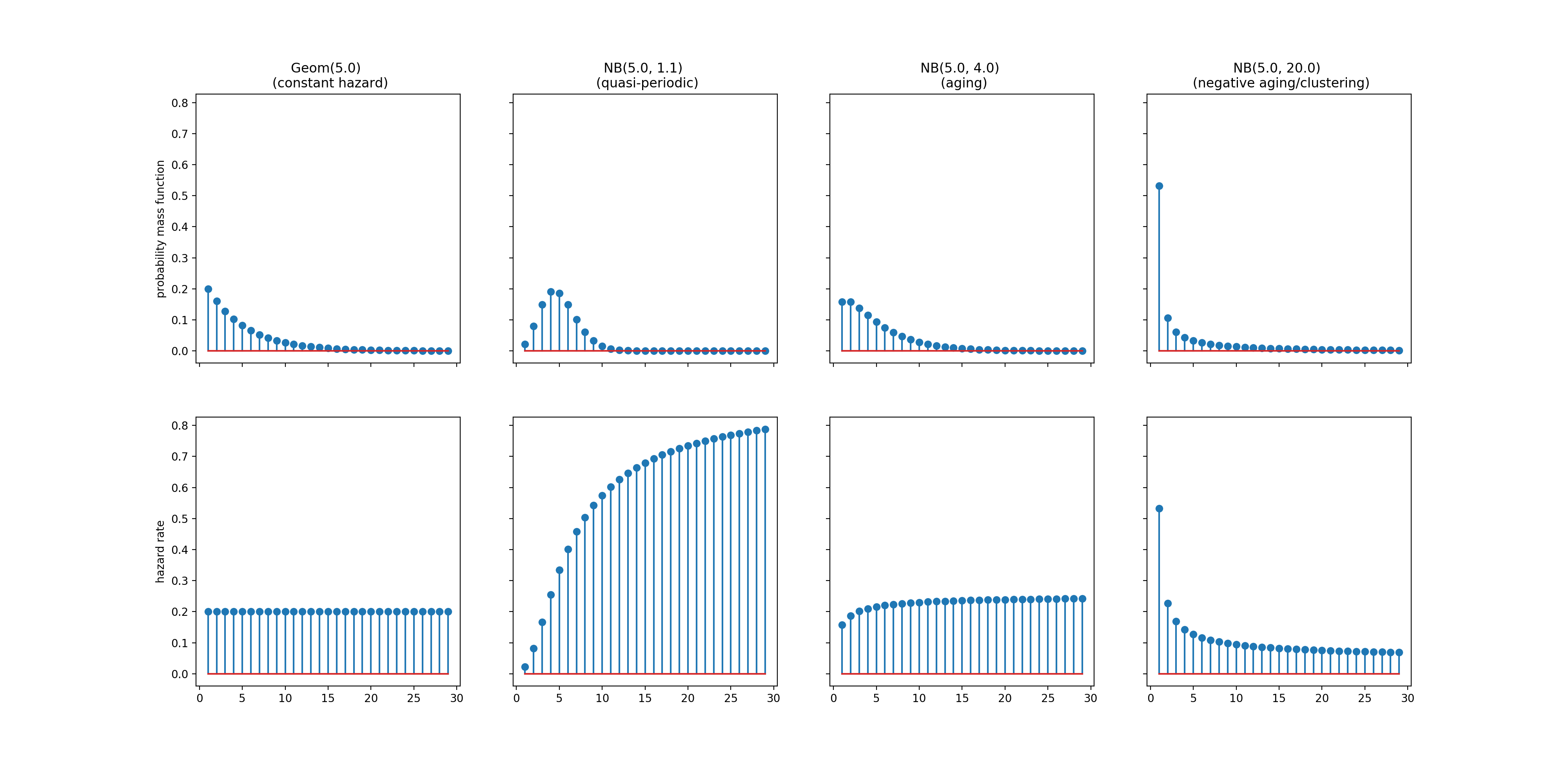}}
	\caption{Comparing probability mass functions and hazard rates of geometric and negative binomial distributions}
	\label{fig:neg_bin_dists}
\end{figure}

Building on this insight, we introduce the {\em Static NB-Po} and {\em Static NB-NB} models, merely by replacing the geometric interarrival time assumption with the negative binomial distribution. 
As in the rest of this work, we will rely on maximum likelihood parameter estimation and the parametric bootstrap (forward sampling) for forecast estimates of these models. 
We should note, however, that {\em Static NB-Po} and {\em Static NB-NB} models yield closed form forecast estimates, which we explore in the appendix.

\subsection*{Self-Modulating Discrete-Time Renewal Processes}

The renewal process construction highlights a potentially unrealistic assumption, that interarrival times are independent.
To the contrary, they have been shown to exhibit autocorrelation \cite{willemain_forecasting_1994}.
Many prior models have either attempted to craft a well-defined probabilistic model for nonstationary and autocorrelated demand arrival regimes, or to encode these relationships via SES-like heuristics \cite{shenstone_stochastic_2005,hyndman_forecasting_2008}.
Two examples, {\em EWMA G-Po} and {\em EWMA G-NB} were given above.

In order to make a similar extension to these models, we first frame them as instances of {\em self-modulating} DTRPs.
Such models, where the interarrival time distribution is determined by a conditional rate function have been widely explored in the TPP literature \cite{daleyintroduction}.
Their key characteristic is that interarrival times now obey a {\em conditional} distribution, the parameters of which are easily computed.
More formally, letting $Q_i$ denote interarrival times as in Definition~\ref{def:dtrp},
\begin{defn}
A sequence of $Q_i$ define a {\em self-modulating DTRP} if the sequence obeys an identically defined conditional distribution
$
\IP{Q_i = k|\HH_i}
$
for all $i, k$, where $\HH_i$ denotes the history\footnotemark\, of events up to interval $i$.
\end{defn}
\footnotetext{More formally, $\HH_i$ denotes a {\em filtration}---a growing sequence of sigma-algebras---to which the $Q_i$ process is {\em adapted}.}

EWMA models are a good example.
For all $i$, the conditional distribution of $Q_i$ is defined as a random variable whose mean is the EWMA of previous realizations.
In other words, the history $\HH_i$ completely determines the distribution of $Q_i$.
A natural next step is to follow the arguments of the previous section, and alter the conditional interarrival time distributions with the negative binomial distribution. 
We will explore {\em EWMA NB-Po} and {\em EWMA NB-NB} models as extensions of their counterparts with geometric interarrival times. 

Note that the {\em convergence problem} that plagues such models is not mitigated by the change in interarrival time distributions. 
For our model specification, the problem takes a slightly different form.
That is, over longer forecast horizons, EWMA model forecasts converge to a forecast of ``all ones,'' introducing a slightly different bias than the one shown in \cite{grunwald_properties_1997}. 
We give further details of this issue in the appendix.
Finally, note that self-modulating DTRP models can be seen as instances of {\em Markov renewal processes} \cite{cinlar_markov_1975}.

\subsection*{RNN-based Discrete-Time Renewal Processes} \label{subsec:dt_deep_renewal}

DTRPs yield an intuitive and simple way to extend IDF models. 
Yet, they suffer from two main limitations that hinder a realistic understanding of intermittent demand series. 
First, the conditional distribution of demand sizes and intervals rely only on the exponentially weighted moving average.
However, the recurrence relation that determines the conditional mean of sizes and intervals may take more complicated functional forms.
The second limitation arises since size and interval processes are assumed independent. 
One could reasonably expect that longer interdemand intervals result from higher demands---\eg, when a customer {\em frontloads} inventory---an effect not captured by this independence assumption.
Finally, other features or covariates that are consequential for accurate forecasts may be available with each issue point, such as dummy variables for discounts or holidays.

To alleviate the first issue, we observe that under mild conditions any recursive form for parameterizing $\hat{Q}_i$ would yield a well-defined self-modulating DTRP.
An RNN is no exception. 
Recall from the {\em EWMA} model definition that 
\[
\hat{Q}_i = \EWMA_\alpha(Q_{1:i}) = f(\hat{Q}_{i-1}, Q_i) = \alpha Q_i + (1 - \alpha) \hat{Q}_{i-1}.
\] 
An LSTM, admits a similar recursion, where instead of $\hat{Q}_i$, we introduce the {\em hidden state} $\mathbf{h}_i \in \IR^H$. 
We define {\em RNN-based} models via the following alternative form,
\begin{align}
\hat{Q}_i &= g_\mathbf{w}(\mathbf{h}_i), \\
\mathbf{h}_i &= \LSTM_\Theta (\mathbf{h}_{i-1}, Q_i),
\end{align}
where $g_\mathbf{w}(\mathbf{h}) = 1 + \log(1 + \exp(\mathbf{w}^\top \mathbf{h} + w_0))$ is a projection from the hidden state to the domain of $\hat{Q}_i$, parameterized by $\mathbf{w} \in \IR^H, w_0 \in \IR$.
By $\LSTM_\Theta (.)$, as given above, we refer to the long short-term memory recurrent neural network function parameterized by weights collected in $\Theta$.

Secondly, LSTM inputs can be extended to include both the previous interdemand time, and the previous demand size, at no significant added cost for estimation or inference. 
Indeed, this is the approach taken by previous IDF models with neural networks \cite{gutierrez_lumpy_2008,kourentzes_intermittent_2013}.
Finally, the LSTM inputs may include any additional covariates that are available, in contrast to other IDF methods that propose no clear way to accommodate them.

Though conceptually similar, EWMA and LSTM recurrences have a fundamental difference. 
The LSTM network has a potentially high-dimensional parameter vector $\Theta$ that has to be fitted. 
Moreover, neural networks are known to be ``data-hungry,'' requiring a large data set to fit as opposed to the parsimonious representation of temporal continuity embedded in exponential smoothing models.
In order to alleviate this issue, when multiple intermittent demand series are available, we share the parameters $\Theta$ across different samples. 
In machine learning terminology, the LSTM representation learned is a {\em global} model, shared across different items (SKUs). See, \eg, \cite{benidis2020neural,januschowski2018deep,januschowski2019criteria}.

We give a summary of the discrete-time renewal process models introduced so far in Table~\ref{tab:dt_models}.

\begin{table}[!ht]
	\caption{Discrete Time Models}
	\label{tab:dt_models}
	\centering
	\begin{adjustwidth}{0in}{0in}
	\noindent\makebox[\textwidth]{
	\begin{tabular}{p{3cm} p{6cm} p{6cm}}
		\toprule
		{\bf Model} 	 & $Q_i$ & $M_i$ \\ \midrule
		Static NB-Po 	& $\mathcal{NB}(\mu_q, \nu_q)$, i.i.d. & $\mathcal{PO}(\mu_m)$, i.i.d. \\[10pt]
		Static NB-NB    & $\mathcal{NB}(\mu_q, \nu_q)$, i.i.d. & $\mathcal{NB}(\mu_m, \nu_m)$, i.i.d.\\[10pt]
		\midrule
		EWMA NB-Po 	 & {$\!
										\begin{aligned}
										&\mathcal{NB}(\hat{Q}_{i-1}, \nu_q)\\
										&\hat{Q}_{i-1} = \EWMA_\alpha(Q_{1:i-1})
										\end{aligned}$
									} &
								 {$\!
								 	\begin{aligned}
								 	&\mathcal{PO}(\hat{M}_{i-1})\\
								 	&\hat{M}_{i-1} = \EWMA_\alpha(M_{1:i-1})
								 	\end{aligned}$
								 }  \\[25pt]
		EWMA NB-NB 	 & {$\!
									\begin{aligned}
									&\mathcal{NB}(\hat{Q}_{i-1}, \nu_q)\\
									&\hat{Q}_{i-1} = \EWMA_\alpha(Q_{1:i-1})
									\end{aligned}$
								} &
								{$\!
									\begin{aligned}
									&\mathcal{NB}(\hat{M}_{i-1}, \nu_m)\\
									&\hat{M}_{i-1} = \EWMA_\alpha(M_{1:i-1})
									\end{aligned}$
								}  \\[25pt]  \midrule
		RNN NB-Po 	 & {$\!
									\begin{aligned}
									&\mathcal{NB}(\hat{Q}_{i-1}, \nu_q)\\
									&\hat{Q}_{i-1} = g\left(\LSTM_\theta(\mathbf{h}_{i-1}, Q_{i-1}, M_{i-1})\right)
									\end{aligned}$
								} &
								{$\!
									\begin{aligned}
									&\mathcal{PO}(\hat{M}_{i-1})\\
									&\hat{M}_{i-1} = g\left(\LSTM_\theta(\mathbf{h}_{i-1}, Q_{i-1}, M_{i-1})\right)
									\end{aligned}$
								}  \\[25pt]
		RNN NB-NB 	 & {$\!
										\begin{aligned}
										&\mathcal{NB}(\hat{Q}_{i-1}, \nu_q)\\
										&\hat{Q}_{i-1} = g\left(\LSTM_\theta(\mathbf{h}_{i-1}, Q_{i-1}, M_{i-1})\right)
										\end{aligned}$
									} &
									{$\!
										\begin{aligned}
										&\mathcal{NB}(\hat{M}_{i-1}, \nu_m)\\
										&\hat{M}_{i-1} = g\left(\LSTM_\theta(\mathbf{h}_{i-1}, Q_{i-1}, M_{i-1})\right)
										\end{aligned}$									
									}  \\[25pt]   \bottomrule
	\end{tabular}
	}
	\end{adjustwidth}
\end{table}

\subsection*{Continuous Time Renewal Processes}

For our last set of models, we explore a connection between the models of the previous section 
and temporal point processes.

When granular timestamps of individual demand events are available, temporal point processes can be used directly as a model for intermittent demand.
A flexible class of continuous-time renewal processes result from the use of RNNs similar to their functions above. 
To make the connection, note that continuous-time renewal processes arise as a limit case of their discrete time counterparts.
For instance, it is well known that the Poisson process---the simplest continuous-time renewal process---arises as a limit case of the Bernoulli process. See, \eg, \cite[Ch. 1]{kingman}.
Similarly, the geometric distribution of interarrival times leads to an exponential distribution in the continuous case. 

Let us introduce processes $\{Q'_j\}, \{M'_j\}$ as continuous-time interarrival time and demand size processes. 
The index $j$ now runs over individual demand events, \eg, purchase orders, and not only positive demand intervals. 
We keep the (conditional) distributions of $M'_j$ identical to the discrete-time model as the support of the random variable and its semantics are identical.
To address continuous time, we simply change $Q'_j$ to a random variable with continuous support. 
We list these models in Table~\ref{tab:ct_models}.
Note that the {\em Static E-Po} and {\em Static E-NB} are just homogeneous Poisson processes with positive integer marks.

\begin{table}[!ht]
	\caption{Continuous Time Models}
	\label{tab:ct_models}
	\centering
	\begin{adjustwidth}{0in}{0in}
	\noindent\makebox[\textwidth]{
		\begin{tabular}{p{3cm} p{6cm} p{6cm}}
			\toprule
			{\bf Model} 	 & $Q'_i$ & $M'_i$ \\ \midrule
			Static E-Po 	& $\mathcal{E}(\mu_{q'})$, i.i.d. & $\mathcal{PO}(\mu_{m'})$, i.i.d. \\[10pt]
			Static E-NB 	& $\mathcal{E}(\mu_{q'})$, i.i.d. & $\mathcal{NB}(\mu_{m'}, \nu_{m'})$, i.i.d. \\[10pt]
			\midrule
			RNN E-Po 	 & {$\!
				\begin{aligned}
				&\mathcal{E}(\hat{Q'}_{i-1})\\
				&\hat{Q'}_{i-1} = g\left(\LSTM_\theta(\mathbf{h}_{i-1}, Q'_{i-1}, M'_{i-1})\right)
				\end{aligned}$
			} &
			{$\!
				\begin{aligned}
				&\mathcal{PO}(\hat{M'}_{i-1})\\
				&\hat{M'}_{i-1} = g\left(\LSTM_\theta(\mathbf{h}_{i-1}, Q'_{i-1}, M'_{i-1})\right)
				\end{aligned}$
			}  \\[25pt] 
			RNN E-NB 	 & {$\!
				\begin{aligned}
				&\mathcal{E}(\hat{Q'}_{i-1})\\
				&\hat{Q'}_{i-1} = g\left(\LSTM_\theta(\mathbf{h}_{i-1}, Q'_{i-1}, M'_{i-1})\right)
				\end{aligned}$
			} &
			{$\!
				\begin{aligned}
				&\mathcal{NB}(\hat{M'}_{i-1})\\
				&\hat{M'}_{i-1} = g\left(\LSTM_\theta(\mathbf{h}_{i-1}, Q'_{i-1}, M'_{i-1})\right)
				\end{aligned}$
			}  \\[25pt] \bottomrule
	\end{tabular}
    }
	\end{adjustwidth}
\end{table}

This approach, combining RNNs with TPP has been taken in the machine learning literature to construct flexible models \cite{du_recurrent_2016,mei_neural_2017,turkmen2019fastpoint,shchur2019intensity}, with connections to full Markov renewal processes explored recently by \cite{sharma2018point}. 
\section*{Experiments} \label{sec:experiments}


\subsection*{Synthetic Data}

We first illustrate our approach on a series of synthetic data sets. 

Our work proposes two main ideas.
First, we cast existing IDF models as instances of renewal processes, naturally extending previous point forecast models to probabilistic forecasting.
Therefore, in contrast to point forecast methods such as Croston's estimate \cCro{}  and the Teunter-Syntetos-Babai (TSB) \cTSB{}  method, renewal processes are able to produce consistent probabilistic forecasts.
Moreover, by providing interdemand times with more flexible distributions, we are able to better capture common intermittent demand patterns such as periodicity and aging.

To highlight this, we simulate perfectly periodic hourly intermittent demand data of 10 weeks, fixing the interdemand period at 20 hours, and drawing demand sizes i.i.d. from a Poisson distribution with a mean of 5.
We illustrate forecasts from previous point forecast methods, Croston \cCro{}, Syntetos-Boylan approximation (SBA) \cSBA{}, and TSB \cTSB{} in the top-most panel in Fig~\ref{fig:synth_1}.
The vertical red line marks the beginning of the forecast horizon, and the different colors in the forecast range refer to point forecasts from different methods.
We obtained these forecasts via the \texttt{tsintermittent} library \cite{kourentzes2014tsintermittent}.
The horizontal red line denotes the x-axis.
The next three panels in Fig~\ref{fig:synth_1} display forecasts from the Static G-Po \cCro{} model as well as the newly introduced Static NB-NB and RNN NB-NB models respectively.
Shaded areas represent 1\%--99\% forecast intervals obtained by probabilistic forecasting methods via forward sampling---or, parametric bootstrapping.
The dark blue line in each graph marks the mean of sampled trajectories.
The RNN used is a single-hidden layer LSTM network with only 5 units, and no hyperparameter optimization is performed.

\begin{figure}[!h]
	\makebox[\textwidth][c]{\includegraphics[width=1.3\textwidth]{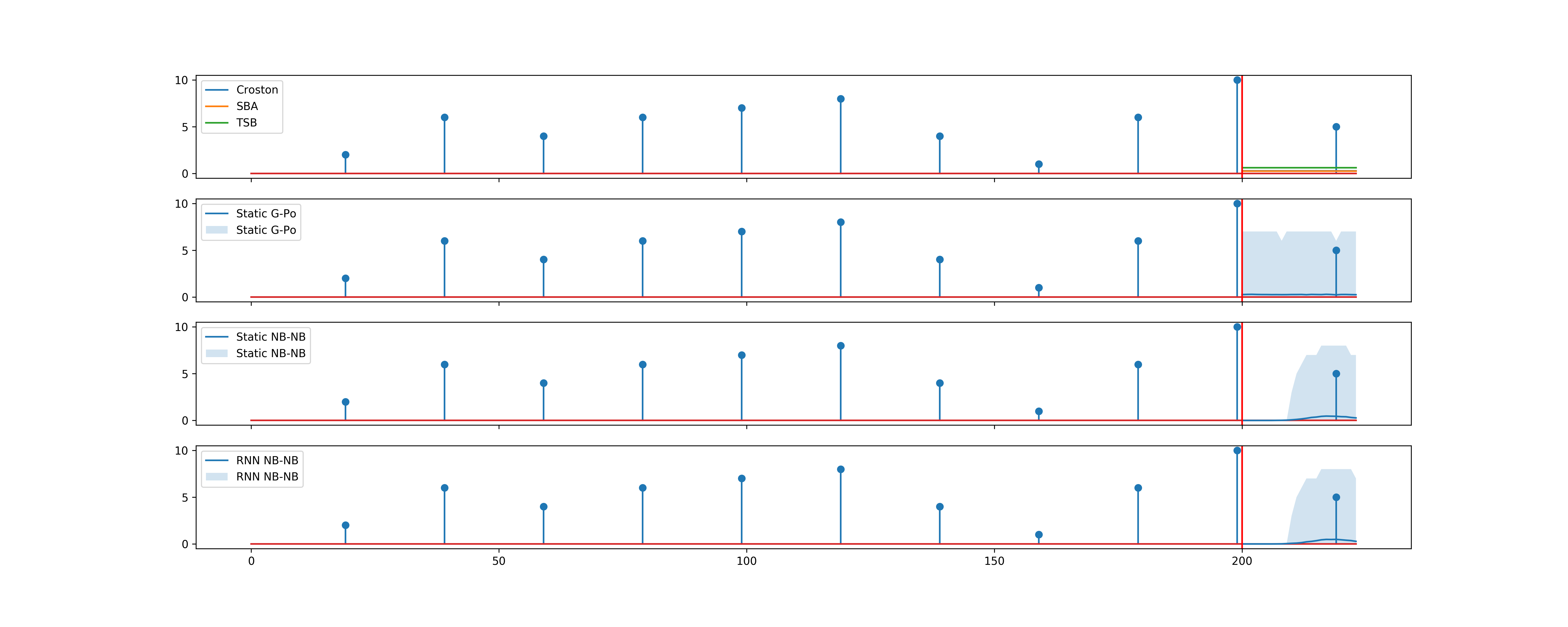}}
	\caption{Comparing forecasts on a perfectly periodic data set. The vertical red line denotes the start of the forecast period.}
	\label{fig:synth_1}
\end{figure}

As expected, Static G-Po \cCro{} produces constant forecasts similar to point forecasting methods.
On the other hand, by using flexible interdemand time distributions, Static NB-NB and RNN NB-NB models are able to capture periodic demand behavior.
Both renewal processes are able to represent periodicity both in the forecast---i.e., the conditional expectation of the time series---and the probabilistic forecast.

Our second contribution is noting that the exponential moving average in Croston-type methods can be replaced with more flexible function families to capture a wider family of patterns.
Above, we posit that neural networks, used as a building block to replace the exponential moving average, can represent most interesting patterns.
To illustrate this, we sample a single time series with ``alternating'' periods, and a constant demand size of 10.
Interdemand times, in turn, alternate between 4 and 16.
In Fig~\ref{fig:synth_2}, we compare point forecasts, Static G-Po, and two new models meant to represent nonstationarities in interdemand times and demand sizes: EWMA NB-NB, and RNN-based NB-NB.

\begin{figure}[!h]
	\makebox[\textwidth][c]{\includegraphics[width=1.3\textwidth]{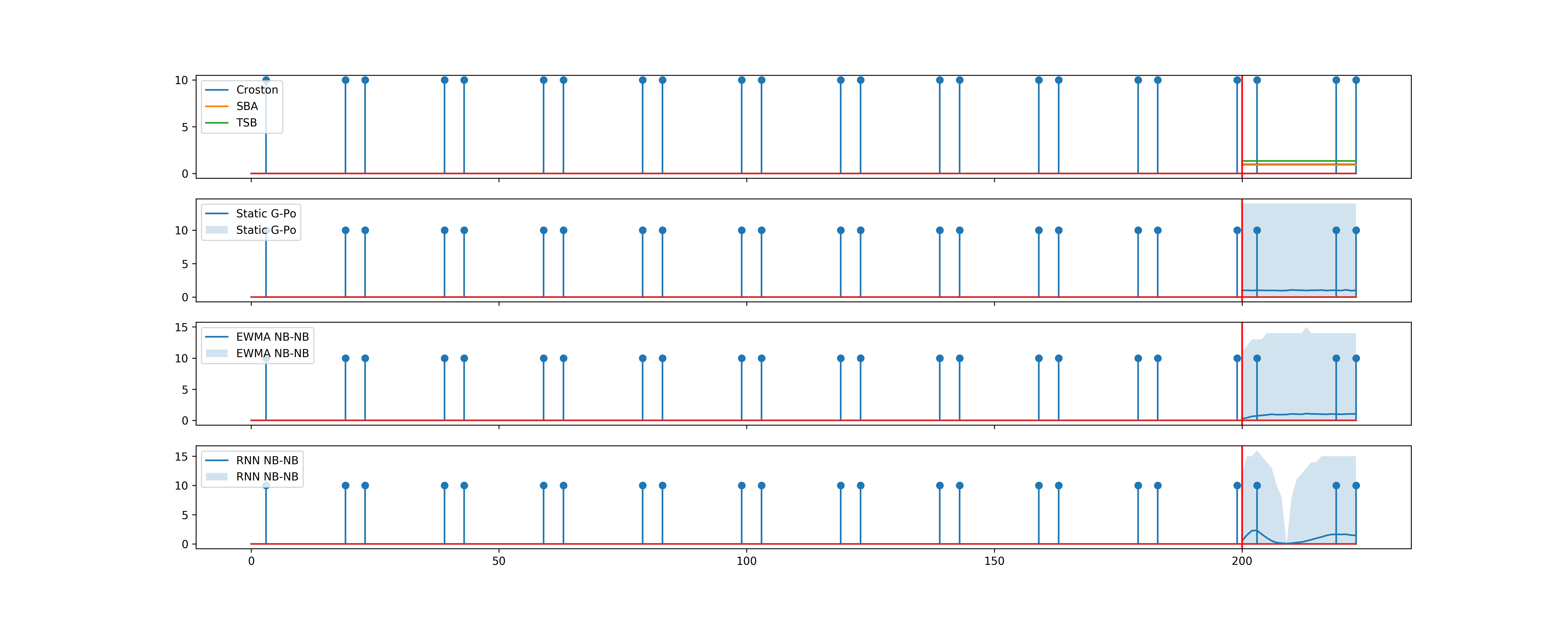}}
	\caption{Comparing forecasts on a data set with alternating periods.}
	\label{fig:synth_2}
\end{figure}

Even in this simple scenario, the forecasts of the RNN-based method clearly outperform those of moving-average based methods, as the former is able to easily learn the alternating pattern.

Finally, we validate this understanding on three synthetic data sets. 
The {\em Periodic} and {\em Alternating} data sets are obtained by repeating the respective simulations above 100 times.
The {\em Random} data set is drawn from the Static G-Po model.
In a sense, this data set represents purely random demand arrivals, with a Bernoulli process governing the arrival process with Poisson demand sizes.

We expect that our models do not yield any significant benefits in the {\em Random} data set.
Instead, we expect that renewal processes outperform baseline methods in the {\em Periodic} data set, and the RNN-based model outperforms others in alternating periods.
We report results of our experiments, where we compare three point forecast baseline methods (Croston, SBA, TSB), with the probabilistic baseline (Static G-Po) and newly introduced models in Table~\ref{tab:results_synth}.
Definitions of accuracy metrics used can be found in the appendix. 

\begin{table}[!ht]
\centering
\caption{Experiments with synthetic data sets}
\label{tab:results_synth}
\begin{tabular}{@{}llccccc@{}}
\toprule
{\bf Data Sets}
& {\bf Model}            & \textbf{P90 Loss} & \textbf{MAPE} & \textbf{sMAPE} & \textbf{RMSE}   & \textbf{RMSSE}   \\ \midrule
\multirow{6}{*}{\textbf{Random}} 
& Croston    \cCro       &  N/A              &  0.875        &  1.966         & 1.535           & 1.484            \\
& SBA    \cSBA           &  N/A              &  0.880        &  1.967         & {\bf 1.534}     & {\bf 1.480}      \\
& TSB    \cTSB           &  N/A              &  0.892        &  1.971         & 1.571           & 1.528            \\
& Static G-Po    \cCro   &  0.792            &  {\bf 0.873 } &  {\bf 1.965}   & 1.537           & 1.488            \\
& Static NB-NB           &  0.796            &  0.876        &  1.966         & 1.536           & 1.493            \\
& RNN NB-NB              &  {\bf 0.752}      &  0.884        &  1.967         & 1.534     & 1.485            \\ \midrule
\multirow{6}{*}{\textbf{Periodic}} 
& Croston    \cCro       &  N/A              &  0.934        &  1.988         & 1.218           & 1.587            \\
& SBA    \cSBA           &  N/A              &  0.936        &  1.988         & 1.218           & 1.586            \\
& TSB    \cTSB           &  N/A              &  0.946        &  1.990         & 1.232           & 1.610            \\
& Static G-Po    \cCro   &  0.460            &  0.933        &  1.987         & 1.219           & 1.589            \\
& Static NB-NB           &  0.419            &  {\bf 0.872}  &  {\bf 1.977}   & {\bf 1.181}     & {\bf 1.536}      \\
& RNN NB-NB              &  {\bf 0.412}      &  0.878        &  1.978         & 1.183           & 1.536      \\ \midrule
\multirow{6}{*}{\textbf{Alternating}} 
& Croston    \cCro       &  N/A              &  0.755        &  1.855         & 2.339           & 1.692            \\
& SBA    \cSBA           &  N/A              &  0.766        &  1.861         & 2.340           & 1.691            \\
& TSB    \cTSB           &  N/A              &  0.768        &  1.862         & 2.361           & 1.709            \\
& Static G-Po    \cCro   &  1.209            &  0.758        &  1.856         & 2.339           & 1.691            \\
& Static NB-NB           &  1.192            &  0.754        &  1.855         & 2.332           & 1.686            \\
& RNN NB-NB              &  {\bf 1.187}      &  {\bf 0.749}  &  {\bf 1.852}   & {\bf 2.331}     & {\bf 1.686}      \\
\bottomrule
\end{tabular}      
\end{table}

Among our results, P90 Loss can be regarded as an accuracy measure of probabilistic forecasts while others refer to point forecasts.
We find that our expectation is reflected in the results, and that the RNN model is able to better represent forecast distributions across the different data sets.

\subsection*{Real Data}

We test our framework on five intermittent demand data sets, comparing extended models to well-known baselines.
These include three well-known standard data sets for IDF and the recently held M5 competition data set which is mostly composed of intermittent series. 
We also introduce a new data set from the UCI Machine Learning repository.\footnotemark{}
We give further details of these data sets below.

\footnotetext{\url{https://archive.ics.uci.edu/ml/index.php}}

\begin{itemize}
\item {\bf Car Parts} consists of monthly parts demand data for a US-based automobile company, previously studied in \cite{hyndman_forecasting_2008,snyder_forecasting_2012}.
It includes 2674 time series.
We conduct experiments on 2503 time series which have complete data for 51 months, with at least one demand point in the first 45 months.

\item The {\bf Auto} data set consists of monthly demand data for 3000 items, for a period of 24 months. 
It was previously used by \cite{kourentzes_intermittent_2014}, and originates from \cite{syntetos_accuracy_2005}.
We report results on 1227 time series that are ``intermittent'' or ``lumpy'' based on the Syntetos-Boylan classification (SBC) \cite{syntetos_accuracy_2005}, using the implementation in \texttt{tsintermittent} \cite{kourentzes2014tsintermittent}.

\item The {\bf RAF} data set consists of aerospace parts demand data taken from the Royal Air Force, used previously by \cite{eaves_forecasting_2004}.
The data set includes 84 months of data for 5000 parts. 

\item The {\bf UCI} data set includes demand records of retail items of 4070 individual items between 01/12/2010 and 09/12/2011 of a UK-based online retailer \cite{chen2012data}.
Most notably, the data set is a transaction record of individual purchases, with exact date and timestamps. 
Among these, we aggregate purchase records of the last three months to daily time series, each corresponding to a single product.
We retain time series with demand in the first 30 days, and with less than 90 purchase records in the entire timespan.
All 1381 such time series are classified as lumpy demand per SBC.

\item The {\bf M5} data set is taken from the recently launched M5 competition \cite{m5guide}.
The data set includes daily unit sales per product and store in Walmart, over the course of over 5 years. 
It contains 30489 time series, with 1913 days of data. 
23401 of these series are classified as intermittent while 5953 are lumpy according to SBC.

\end{itemize}

Summary statistics of data sets are given in Table~\ref{tab:dataset_Stats}, where $\text{CV}^2$ denotes the squared coefficient of variation across all demand instances in the data set.

\begin{table}[!ht]
\centering
\caption{Summary statistics of data sets}
\label{tab:dataset_Stats}
\resizebox{\textwidth}{!}{
\begin{tabular}{lp{2cm}p{2cm}p{2cm}p{2cm}p{2cm}}
\toprule
                                       &{\bf Car Parts} & {\bf Auto}  & {\bf RAF}   & {\bf UCI}   & {\bf M5}    \\ \midrule
Number of Time Series - $M$            & 2503           & 1227        & 5000        & 1381        & 30489       \\
Time Series Length - $N$               & 51             & 24          & 84          & 90          & 1913        \\
Demand Size - Mean                     & 2.02           & 3.79        & 14.19       & 5.21        & 3.54        \\
Demand Size - $\text{CV}^2$            & 0.86           & 4.42        & 11.89       & 3.28        & 3.08        \\
Mean Interdemand Time - $p$            & 3.41           & 1.50        & 8.75        & 4.89        & 3.13        \\
$\text{CV}^2$ of Interdemand Time      & 1.93           & 0.22        & 0.67        & 1.72        & 74.3        \\
Total Number of Issue Points           & 32,093         & 56,133      & 42,695      & 23,416      & 18,549,855  \\
Mean Number of Issue Points            & 12.82          & 18.71       & 8.54        & 16.96       & 608.4       \\
Number of ``Intermittent'' Items (SBC) & 2087           & 942         & 2597        & 0           & 23041       \\
Number of ``Lumpy'' Items (SBC)        & 412            & 285         & 2403        & 1381        & 5953        \\
\bottomrule
\end{tabular}
}
\end{table}

We take the last 6 demand periods of each time series as a held out sample, and train models on the preceding periods.
Results are reported on the forecast accuracy of the last 6 periods, \ie, with a 6-step-ahead forecast.
For the M5 competition, both the held-out sample and the forecast horizon are set to 28 days, in accordance with competition rules \cite{m5guide}.

We implement all models on Apache MXNet \cite{chen2015mxnet}, and rely on the Gluon Adam optimizer for parameter fitting wherever necessary \cite{kingma2014adam}.
We set the learning rate to 0.1. 
We run experiments on the Amazon SageMaker platform, cf. \cite{januschowski2018now,liberty2019}.
For Croston-type models, we set $\alpha=0.1$.

Notably, due to the lack of sufficiently large data for cross-validation, we perform no hyperparameter optimization for the neural network architecture, regularization, or optimization parameters.
For RNN-based models, we use a single hidden layer of 20 units in the LSTM; and a two layer LSTM for the M5 data set.
The LSTM output is mapped to the parameter domains via a softplus function.
For regularization, we rely on {\em weight decay}, setting the decay rate to 0.01.

We evaluate performance of both forecast distributions and point forecasts.
As previously noted, forecast accuracy metrics are especially elusive in the case of IDF.
We report results on a wide variety of metrics which have been used in both IDF and in forecasting in general.
Definitions of these metrics are given in the appendix.
We repeat each experiment three times, and report the average and standard deviation of metrics.
For probabilistic models, we evaluate forecasts based on 250 trajectories sampled from the fitted model.
For point forecast comparisons, we compare true values with means of sampled trajectories.

Discrete-time model results are reported in Tables~\ref{tab:results_dtrp} and \ref{tab:results_dtrp_2}, where we report the means and standard deviations of all accuracy measures. 
\emph{All zeros} refers to accuracy obtained by a point forecast of 0s across the forecast horizon.
Across metrics, lower numbers indicate better predictive accuracy.
Our results vary across model families and data sets, however some key themes emerge.

\begin{sidewaystable}
\centering
\caption{Experiment results for discrete-time models with Car Parts and Auto data sets}
\label{tab:results_dtrp}
\resizebox{\textwidth}{!}{
\begin{tabular}{p{3cm}p{8cm}cccccc}
\toprule
Data Set  & Model         &     P50 Loss            &         P90 Loss        &  MAPE                      &             sMAPE         &   RMSE                        & RMSSE                       \\
\midrule
\multirow{16}{*}{\textbf{Car Parts}}
& All Zeros               &  N/A                    &  N/A                    &  1.000                     &  2.000                    &   1.513                       & 1.406                       \\
& Croston    \cCro        &  N/A                    &  N/A                    &  0.568                     &  {\bf 1.604              }&   1.372                       & 1.307                       \\
& SBA        \cSBA        &  N/A                    &  N/A                    &  0.571                     &  1.607                    &   1.368                       & 1.301                       \\
& TSB        \cTSB        &  N/A                    &  N/A                    &  0.580                     &  1.615                    &   1.362                       & 1.278                       \\
& Static G-Po    \cCro    &  0.750 $\pm$ 0.002      &  0.638 $\pm$ 0.000      &  0.583 $\pm$ 0.000         &  1.608 $\pm$ 0.000        &   1.410 $\pm$ 0.001           & 1.345 $\pm$ 0.000           \\
& Static G-NB    \cSny    &  0.718 $\pm$ 0.001      &  0.713 $\pm$ 0.001      &  {\bf 0.538 $\pm$ 0.001   }&  {\bf 1.583 $\pm$ 0.000  }&   1.434 $\pm$ 0.000           & 1.366 $\pm$ 0.001           \\
& EWMA G-Po      \cShen   &  0.716 $\pm$ 0.002      &  0.791 $\pm$ 0.004      &  0.634 $\pm$ 0.001         &  1.663 $\pm$ 0.000        &   1.438 $\pm$ 0.001           & 1.363 $\pm$ 0.000           \\
& EWMA G-NB               &  0.711 $\pm$ 0.001      &  0.800 $\pm$ 0.000      &  0.635 $\pm$ 0.001         &  1.663 $\pm$ 0.000        &   1.438 $\pm$ 0.000           & 1.363 $\pm$ 0.001           \\
& Static NB-Po            &  0.663 $\pm$ 0.001      &  {\it 0.553 $\pm$ 0.002}&  0.625 $\pm$ 0.001         &  1.633 $\pm$ 0.001        &   {\it 1.343 $\pm$ 0.002     }& 1.257 $\pm$ 0.001           \\
& Static NB-NB            &  0.663 $\pm$ 0.000      &  0.620 $\pm$ 0.003      &  0.580 $\pm$ 0.001         &  1.605 $\pm$ 0.000        &   1.365 $\pm$ 0.001           & 1.290 $\pm$ 0.001           \\
& EWMA NB-Po              &  0.720 $\pm$ 0.001      &  0.751 $\pm$ 0.005      &  0.649 $\pm$ 0.001         &  1.668 $\pm$ 0.000        &   1.435 $\pm$ 0.002           & 1.358 $\pm$ 0.001           \\
& EWMA NB-NB              &  0.708 $\pm$ 0.001      &  0.764 $\pm$ 0.002      &  0.649 $\pm$ 0.001         &  1.668 $\pm$ 0.001        &   1.436 $\pm$ 0.002           & 1.358 $\pm$ 0.001           \\
& RNN G-Po                &  0.396 $\pm$ 0.005      &  {\bf 0.447 $\pm$ 0.005}&  {\bf 0.544 $\pm$ 0.005}   &  1.759 $\pm$ 0.002        &   {\bf 1.062 $\pm$ 0.004}     & {\bf 0.996 $\pm$ 0.006}     \\
& RNN G-NB                &  0.401 $\pm$ 0.006      &  0.452 $\pm$ 0.009      &  0.550 $\pm$ 0.016         &  1.761 $\pm$ 0.007        &   1.074 $\pm$ 0.005           & {\bf 1.015 $\pm$ 0.008}     \\
& RNN NB-Po               &  {\bf 0.388 $\pm$ 0.002}&  0.458 $\pm$ 0.003      &  0.561 $\pm$ 0.006         &  1.763 $\pm$ 0.003        &   {\bf 1.068 $\pm$ 0.001}     & 1.021 $\pm$ 0.008           \\
& RNN NB-NB               &  {\bf 0.385 $\pm$ 0.001}&  {\bf 0.448 $\pm$ 0.004}&  0.565 $\pm$ 0.036         &  1.765 $\pm$ 0.014        &   {\bf 1.068 $\pm$ 0.004}     & 1.026 $\pm$ 0.027           \\
\cline{1-8}
\multirow{16}{*}{\textbf{Auto}}
& All Zeros               &  N/A                    &  N/A                    &  1.000                     &  2.000                    &   9.618                       & 1.820                       \\
& Croston    \cCro        &  N/A                    &  N/A                    &  0.536                     &  1.077                    &   7.881                       & 1.477                       \\
& SBA        \cSBA        &  N/A                    &  N/A                    &  {\bf 0.529               }&  1.081                    &   7.868                       & {\bf 1.468                 }\\
& TSB        \cTSB        &  N/A                    &  N/A                    &  0.552                     &  1.079                    &   8.000                       & 1.501                       \\
& Static G-Po    \cCro    &  2.383 $\pm$ 0.002      &  {\bf 1.527 $\pm$ 0.001}&  0.540 $\pm$ 0.001         &  1.074 $\pm$ 0.000        &   {\bf 7.842 $\pm$ 0.004     }& 1.479 $\pm$ 0.001           \\
& Static G-NB    \cSny    &  {\bf 2.121 $\pm$ 0.000}&  2.277 $\pm$ 0.004      &  0.755 $\pm$ 0.000         &  1.136 $\pm$ 0.000        &   8.439 $\pm$ 0.003           & 1.828 $\pm$ 0.002           \\
& EWMA G-Po      \cShen   &  2.503 $\pm$ 0.008      &  1.574 $\pm$ 0.002      &  0.586 $\pm$ 0.000         &  1.088 $\pm$ 0.000        &   8.144 $\pm$ 0.010           & 1.525 $\pm$ 0.000           \\
& EWMA G-NB               &  2.330 $\pm$ 0.002      &  1.605 $\pm$ 0.001      &  0.586 $\pm$ 0.000         &  1.088 $\pm$ 0.000        &   8.134 $\pm$ 0.005           & 1.526 $\pm$ 0.000           \\
& Static NB-Po            &  2.335 $\pm$ 0.003      &  {\bf 1.521 $\pm$ 0.001}&  {\bf 0.516 $\pm$ 0.000   }&  {\it 1.066 $\pm$ 0.000  }&   {\bf 7.809 $\pm$ 0.023     }& {\bf 1.452 $\pm$ 0.001}     \\
& Static NB-NB            &  {\bf 2.109 $\pm$ 0.002}&  2.271 $\pm$ 0.006      &  0.743 $\pm$ 0.001         &  1.136 $\pm$ 0.000        &   8.465 $\pm$ 0.009           & 1.809 $\pm$ 0.000           \\
& EWMA NB-Po              &  2.493 $\pm$ 0.001      &  1.575 $\pm$ 0.001      &  0.585 $\pm$ 0.000         &  1.088 $\pm$ 0.000        &   8.156 $\pm$ 0.016           & 1.524 $\pm$ 0.000           \\
& EWMA NB-NB              &  2.333 $\pm$ 0.003      &  1.606 $\pm$ 0.001      &  0.586 $\pm$ 0.001         &  1.089 $\pm$ 0.001        &   8.156 $\pm$ 0.008           & 1.526 $\pm$ 0.000           \\
& RNN G-Po                &  2.395 $\pm$ 0.020      &  1.649 $\pm$ 0.010      &  0.548 $\pm$ 0.014         &  {\bf 1.061 $\pm$ 0.001  }&   8.725 $\pm$ 0.016           & 1.492 $\pm$ 0.008           \\
& RNN G-NB                &  2.158 $\pm$ 0.028      &  1.834 $\pm$ 0.014      &  0.674 $\pm$ 0.036         &  1.076 $\pm$ 0.004        &   8.814 $\pm$ 0.053           & 1.630 $\pm$ 0.031           \\
& RNN NB-Po               &  2.411 $\pm$ 0.036      &  1.649 $\pm$ 0.005      &  0.554 $\pm$ 0.020         &  {\bf 1.062 $\pm$ 0.001  }&   8.634 $\pm$ 0.018           & 1.497 $\pm$ 0.016           \\
& RNN NB-NB               &  2.119 $\pm$ 0.009      &  1.892 $\pm$ 0.004      &  0.636 $\pm$ 0.013         &  1.076 $\pm$ 0.006        &   8.951 $\pm$ 0.015           & 1.605 $\pm$ 0.014           \\
\cline{1-8}
\multirow{16}{*}{\textbf{RAF}}
& All Zeros               &  N/A                    &  N/A                    &  1.000                     &  2.000                    &   16.786                      & {\bf 1.423                 }\\
& Croston    \cCro        &  N/A                    &  N/A                    &  0.845                     &  1.949                    &   16.545                      & 1.777                       \\
& SBA        \cSBA        &  N/A                    &  N/A                    &  0.846                     &  1.950                    &   {\bf 16.543                }& 1.770                       \\
& TSB        \cTSB        &  N/A                    &  N/A                    &  0.881                     &  1.963                    &   17.731                      & {\bf 1.710                 }\\
& Static G-Po    \cCro    &  1.312 $\pm$ 0.000      &  2.878 $\pm$ 0.008      &  0.842 $\pm$ 0.004         &  1.948 $\pm$ 0.000        &   16.557 $\pm$ 0.019          & 1.818 $\pm$ 0.000           \\
& Static G-NB    \cSny    &  1.287 $\pm$ 0.000      &  {\bf 2.354 $\pm$ 0.001}&  {\bf 0.726 $\pm$ 0.003   }&  {\bf 1.936 $\pm$ 0.000  }&   {\bf 16.541 $\pm$ 0.018    }& 1.922 $\pm$ 0.001           \\
& EWMA G-Po      \cShen   &  1.339 $\pm$ 0.000      &  3.094 $\pm$ 0.006      &  0.863 $\pm$ 0.001         &  1.945 $\pm$ 0.000        &   16.657 $\pm$ 0.015          & 1.915 $\pm$ 0.000           \\
& EWMA G-NB               &  1.320 $\pm$ 0.000      &  2.782 $\pm$ 0.005      &  0.863 $\pm$ 0.002         &  1.945 $\pm$ 0.000        &   16.663 $\pm$ 0.007          & 1.925 $\pm$ 0.001           \\
& Static NB-Po            &  1.324 $\pm$ 0.000      &  3.299 $\pm$ 0.002      &  0.816 $\pm$ 0.002         &  1.941 $\pm$ 0.000        &   16.607 $\pm$ 0.015          & 1.956 $\pm$ 0.000           \\
& Static NB-NB            &  {\it 1.285 $\pm$ 0.000}&  2.448 $\pm$ 0.000      &  {\bf 0.665 $\pm$ 0.002   }&  {\bf 1.926 $\pm$ 0.000  }&   16.596 $\pm$ 0.021          & 2.099 $\pm$ 0.001           \\
& EWMA NB-Po              &  1.299 $\pm$ 0.004      &  3.118 $\pm$ 0.003      &  0.856 $\pm$ 0.001         &  1.945 $\pm$ 0.000        &   16.666 $\pm$ 0.013          & 1.912 $\pm$ 0.001           \\
& EWMA NB-NB              &  1.291 $\pm$ 0.002      &  2.786 $\pm$ 0.006      &  0.859 $\pm$ 0.002         &  1.946 $\pm$ 0.000        &   16.668 $\pm$ 0.023          & 1.921 $\pm$ 0.001           \\
& RNN G-Po                &  {\bf 1.277 $\pm$ 0.001}&  2.592 $\pm$ 0.021      &  0.798 $\pm$ 0.005         &  1.946 $\pm$ 0.001        &   16.625 $\pm$ 0.007          & 1.825 $\pm$ 0.012           \\
& RNN G-NB                &  {\bf 1.278 $\pm$ 0.001}&  {\bf 2.407 $\pm$ 0.023}&  0.799 $\pm$ 0.002         &  1.949 $\pm$ 0.000        &   16.630 $\pm$ 0.006          & 1.763 $\pm$ 0.002           \\
& RNN NB-Po               &  {\bf 1.278 $\pm$ 0.002}&  2.893 $\pm$ 0.029      &  0.778 $\pm$ 0.005         &  1.938 $\pm$ 0.001        &   16.595 $\pm$ 0.010          & 1.964 $\pm$ 0.018           \\
& RNN NB-NB               &  {\bf 1.278 $\pm$ 0.002}&  2.499 $\pm$ 0.005      &  0.776 $\pm$ 0.004         &  1.943 $\pm$ 0.001        &   16.614 $\pm$ 0.004          & 1.837 $\pm$ 0.006           \\
\bottomrule
\end{tabular}
}
\end{sidewaystable}

\begin{sidewaystable}
\centering
\caption{Experiment results for discrete-time models with UCI and M5 sets}
\label{tab:results_dtrp_2}
\resizebox{\textwidth}{!}{
\begin{tabular}{p{3cm}p{8cm}cccccc}
\toprule
Data Set  & Model         &     P50 Loss            &         P90 Loss        &  MAPE                      &             sMAPE         &   RMSE                        & RMSSE                       \\
\midrule
\multirow{16}{*}{\textbf{UCI}}
& All Zeros               &  N/A                    &  N/A                    &  1.000                     &  2.000                    &   7.661                       & {\bf 3.413                 }\\
& Croston    \cCro        &  N/A                    &  N/A                    &  0.714                     &  1.735                    &   7.456                       & 3.466                       \\
& SBA        \cSBA        &  N/A                    &  N/A                    &  0.709                     &  1.737                    &   7.460                       & 3.462                       \\
& TSB        \cTSB        &  N/A                    &  N/A                    &  0.733                     &  1.730                    &   {\it 7.430                 }& {\bf 3.409                 }\\
& Static G-Po    \cCro    &  2.059 $\pm$ 0.000      &  3.043 $\pm$ 0.006      &  0.714 $\pm$ 0.002         &  1.739 $\pm$ 0.000        &   7.488 $\pm$ 0.001           & 3.514 $\pm$ 0.002           \\
& Static G-NB    \cSny    &  {\it 2.054 $\pm$ 0.000}&  {\it 3.041 $\pm$ 0.003}&  0.655 $\pm$ 0.002         &  {\it 1.729 $\pm$ 0.001  }&   7.458 $\pm$ 0.001           & 3.467 $\pm$ 0.002           \\
& EWMA G-Po      \cShen   &  {\it 2.054 $\pm$ 0.001}&  3.079 $\pm$ 0.005      &  0.719 $\pm$ 0.001         &  1.732 $\pm$ 0.000        &   7.490 $\pm$ 0.001           & 3.548 $\pm$ 0.002           \\
& EWMA G-NB               &  {\bf 2.043 $\pm$ 0.001}&  3.126 $\pm$ 0.004      &  0.726 $\pm$ 0.001         &  1.735 $\pm$ 0.000        &   7.494 $\pm$ 0.001           & 3.555 $\pm$ 0.000           \\
& Static NB-Po            &  2.059 $\pm$ 0.000      &  3.121 $\pm$ 0.004      &  0.698 $\pm$ 0.001         &  1.757 $\pm$ 0.000        &   7.526 $\pm$ 0.002           & 3.585 $\pm$ 0.001           \\
& Static NB-NB            &  2.059 $\pm$ 0.000      &  3.249 $\pm$ 0.001      &  {\it 0.667 $\pm$ 0.001   }&  1.749 $\pm$ 0.000        &   7.504 $\pm$ 0.001           & 3.596 $\pm$ 0.003           \\
& EWMA NB-Po              &  2.120 $\pm$ 0.002      &  3.068 $\pm$ 0.004      &  0.731 $\pm$ 0.001         &  {\it 1.729 $\pm$ 0.000  }&   7.501 $\pm$ 0.002           & 3.568 $\pm$ 0.001           \\
& EWMA NB-NB              &  2.061 $\pm$ 0.000      &  3.113 $\pm$ 0.006      &  0.736 $\pm$ 0.000         &  1.731 $\pm$ 0.000        &   7.503 $\pm$ 0.001           & 3.574 $\pm$ 0.001           \\
& RNN G-Po                &  2.071 $\pm$ 0.007      &  {\bf 2.906 $\pm$ 0.004}&  0.676 $\pm$ 0.006         &  {\bf 1.699 $\pm$ 0.002  }&   {\bf 7.393 $\pm$ 0.004     }& 3.562 $\pm$ 0.010           \\
& RNN G-NB                &  {\bf 2.039 $\pm$ 0.001}&  {\bf 2.870 $\pm$ 0.004}&  0.669 $\pm$ 0.011         &  {\bf 1.686 $\pm$ 0.003  }&   {\bf 7.368 $\pm$ 0.004     }& 3.566 $\pm$ 0.012           \\
& RNN NB-Po               &  2.086 $\pm$ 0.005      &  2.937 $\pm$ 0.005      &  {\bf 0.637 $\pm$ 0.004   }&  1.702 $\pm$ 0.002        &   7.423 $\pm$ 0.004           & 3.573 $\pm$ 0.001           \\
& RNN NB-NB               &  2.062 $\pm$ 0.003      &  2.923 $\pm$ 0.004      &  {\bf 0.618 $\pm$ 0.005   }&  {\bf 1.699 $\pm$ 0.003  }&   7.417 $\pm$ 0.004           & 3.539 $\pm$ 0.018           \\
\cline{1-8}
\multirow{16}{*}{\textbf{M5}}
& All Zeros               &  N/A                    &  N/A                    &  1.000                &  2.000                &  3.852                  & 1.730                       \\
& Croston    \cCro        &  N/A                    &  N/A                    &  0.596                &  1.429                &  2.299                  & 1.370                       \\
& SBA        \cSBA        &  N/A                    &  N/A                    &  0.590                &  1.436                &  {\bf 2.280}              & 1.364                       \\
& TSB        \cTSB        &  N/A                    &  N/A                    &  0.598                &  1.430                &  {\bf 2.213}              & {\bf 1.323}                   \\
& Static G-Po    \cCro    &  1.093 $\pm$ 0.000      &  0.730 $\pm$ 0.000      &  0.616 $\pm$ 0.000      &  1.471 $\pm$ 0.000      &  2.453 $\pm$ 0.000        & 1.414 $\pm$ 0.000             \\
& Static G-NB    \cSny    &  1.250 $\pm$ 0.000      &  1.192 $\pm$ 0.001      &  {\bf 0.532 $\pm$ 0.000}&  1.458 $\pm$ 0.000      &  3.255 $\pm$ 0.001        & 1.520 $\pm$ 0.000             \\
& EWMA G-Po      \cShen   &  {\bf 1.031 $\pm$ 0.000}&  0.667 $\pm$ 0.000      &  0.593 $\pm$ 0.000      &  1.419 $\pm$ 0.000      &  2.301 $\pm$ 0.000        & 1.377 $\pm$ 0.000             \\
& EWMA G-NB               &  1.015 $\pm$ 0.000      &  0.664 $\pm$ 0.000      &  0.593 $\pm$ 0.000      &  1.419 $\pm$ 0.000      &  2.301 $\pm$ 0.000        & 1.378 $\pm$ 0.000             \\
& Static NB-Po            &  1.378 $\pm$ 0.000      &  0.786 $\pm$ 0.000      &  0.615 $\pm$ 0.000      &  1.531 $\pm$ 0.000      &  3.040 $\pm$ 0.002        & 1.529 $\pm$ 0.000             \\
& Static NB-NB            &  1.384 $\pm$ 0.000      &  1.550 $\pm$ 0.001      &  0.569 $\pm$ 0.000      &  1.514 $\pm$ 0.000      &  3.578 $\pm$ 0.001        & 1.728 $\pm$ 0.000             \\
& EWMA NB-Po              &  1.055 $\pm$ 0.000      &  0.663 $\pm$ 0.000      &  0.593 $\pm$ 0.000      &  {\bf 1.402 $\pm$ 0.000}&  2.306 $\pm$ 0.001        & 1.398 $\pm$ 0.000             \\
& EWMA NB-NB              &  1.032 $\pm$ 0.000      &  {\bf 0.661 $\pm$ 0.000}&  0.593 $\pm$ 0.000      &  1.403 $\pm$ 0.000      &  2.308 $\pm$ 0.000        & 1.399 $\pm$ 0.000             \\
& RNN G-Po                &  {\bf 1.028 $\pm$ 0.016}&  {\bf 0.654 $\pm$ 0.006}&  {\bf 0.585 $\pm$ 0.014}&  1.421 $\pm$ 0.008      &  2.387 $\pm$ 0.046        & {\bf 1.360 $\pm$ 0.004}       \\
& RNN G-NB                &  1.111 $\pm$ 0.059      &  0.916 $\pm$ 0.152      &  0.625 $\pm$ 0.051      &  {\bf 1.399 $\pm$ 0.005}&  3.025 $\pm$ 0.234        & 1.512 $\pm$ 0.108             \\
& RNN NB-Po               &  1.041 $\pm$ 0.017      &  0.675 $\pm$ 0.009      &  0.586 $\pm$ 0.004      &  1.444 $\pm$ 0.013      &  2.412 $\pm$ 0.047        & 1.372 $\pm$ 0.005             \\
& RNN NB-NB               &  1.063 $\pm$ 0.035      &  0.721 $\pm$ 0.030      &  0.596 $\pm$ 0.044      &  1.433 $\pm$ 0.011      &  2.644 $\pm$ 0.114        & 1.387 $\pm$ 0.018             \\
\bottomrule
\end{tabular}
}
\end{sidewaystable}

We observe that in four of five data sets, RNN-based models lead to more accurate forecast distributions.  
The exception is the Auto data set, which has a total time series length of 18 steps in sample, barely enough to include any interesting temporal patterns. 
In the RAF data set, while RNN models result in better forecast distributions, as measured by P50 and P90 Loss, they lead to slightly less accurate point forecasts.

In the first three data sets, flexible interdemand times improve both probabilistic and point forecasts.
This confirms our intuition that improvements by more flexible renewal-type models are data set dependent.
The Car Parts data set, which has the lowest variation in demand sizes but a high variation in interdemand intervals yields clear evidence in favor of using DTRPs, and matches our intuition for where such models would be useful.
On the other hand, we find scant evidence that flexible negative binomial demand sizes improve forecast accuracy. 
Again, we observe that it leads to some improvement in the UCI and RAF data sets. 
As both data sets have significantly high variation in demand sizes, this finding correlates with our expectation.

We can also use our results to compare point forecast methods (e.g., Croston, SBA) to their model based counterparts.
We find that similar numbers are obtained, and the result of this comparison depends largely on the accuracy metric used.
For example, MAPE and sMAPE are generally lower for probabilistic forecasts while RMSSE is categorically higher.
We believe this difference is due to the stochastic nature of point forecasts obtained from model-based methods via forward sampling, \ie, that some accuracy measures are more sensitive to the variation introduced by sampling than others.

Different size and interval distributions can lead to variation within model classes, \eg, within RNN-based models. 
Moreover, the direction of this variation is different across data sets.
While this supports our earlier conclusion, that the efficacy of flexible distributions is data set-dependent, we also believe this is partly due to no hyperparameter optimization being performed for RNN-based models.
Indeed, in a real-world scenario, both regularization and training length would be changed for different model configurations.
Here, by keeping them constant to report the ``bare minimum'' of what RNN-based models can achieve in IDF, we inadvertently introduce some noise to the results obtained.

Our results so far lead us to conclude that while our proposed models generally lead to improved forecasts, the exact model configuration depends largely on the data set.
However, Tables~\ref{tab:results_dtrp} and \ref{tab:results_dtrp_2} do not paint a clear picture of which modeling direction is more promising.\footnote{This is in line with other empirical studies on forecasting methods, e.g.,~\cite{alexandrov2019gluonts}, where no overall dominant model is found for forecasting--- unlike other areas of machine learning such as natural language processing, where dominant models have emerged (e.g.,~\cite{devlin2018bert}).}
In order to gauge the improvement brought by our two main modeling ideas, we calculate the significance of improvements brought by model families across different data sets.
For this, we compute the ratio of losses for each model to the loss of the Croston (Static G-Po) model, and report averages of these ratios across data sets.
In Table~\ref{tab:pvals}, we report these average ratios for model families, and the significance level of a one-sample one-sided $t$-test under the null hypothesis that the ratio is equal to or greater than one---\ie, the model brings no improvement.

\begin{table}[!ht]
\centering
\caption{Ratios of model families to the Static G-Po baseline model}
\label{tab:pvals}
\resizebox{\textwidth}{!}{
\begin{tabular}{p{6cm}p{1.5cm}p{1.5cm}p{1.5cm}p{1.5cm}p{1.5cm}p{1.5cm}}
\toprule
{\bf Model Feature}         &     
{\bf P50 Loss}            &         {\bf P90 Loss}        &  {\bf MAPE}             & {\bf sMAPE}              &   {\bf RMSE}                 & {\bf RMSSE}           \\ \midrule
Flexible Demand Size Distribution \newline (*-NB Models) 
& \tril{0.952}{0.000*}  & \tril{1.092}{---} & \tril{1.012}{---}  & \tril{1.006}{---}   & \tril{1.033}{---}      & \tril{1.030}{---}  \\ \midrule
Flexible Interdemand Time Distribution \newline (NB-* Models) 
& \tril{0.963}{0.000*}  & \tril{1.051}{---}  & \tril{1.002}{---}  & \tril{1.007}{---}   & \tril{1.018}{---}      & \tril{1.018}{---}  \\ \midrule
EWMA 
& \tril{0.983}{0.000*}  & \tril{1.043}{---}  & \tril{1.038}{---}  & \tril{1.000}{---}   & \tril{1.001}{---}      & \tril{1.019}{---}  \\ \midrule
RNN
& \tril{0.890}{0.000*}  & \tril{0.948}{0.000*}  & \tril{0.977}{0.001*}  & \tril{1.006}{---}   & \tril{0.987}{0.104}      & \tril{0.966}{0.000*}  \\ \midrule
RNN \& Flexible Demand Size
& \tril{0.882}{0.000*}  & \tril{0.972}{0.105}  & \tril{0.999}{0.472}  & \tril{1.007}{---}   & \tril{1.008}{---}      & \tril{0.976}{0.030*}  \\ \midrule
RNN \& Flexible Interdemand Time
& \tril{0.887}{0.000*}  & \tril{0.947}{0.001*}  & \tril{0.965}{0.000*}  & \tril{1.009}{---}   & \tril{0.980}{0.072}      & \tril{0.968}{0.004*}  \\
\bottomrule
\end{tabular}
}
\end{table}

Here, we find that flexible demand size distribution, proposed in \cite{snyder_forecasting_2012} do not result in improved forecasts across data sets.
Nor do flexible interdemand time distributions, or the renewal process idea, proposed in this paper.
Furthermore, we do not find that exponential moving average-based models result in any significant improvement.
Comparing Table~\ref{tab:pvals} to dataset results supports our claim that while these three model components may improve individual data set performance, they do not do so in general.
However, RNN-based models result in decidedly better forecasts, especially when used in combination with the renewal process idea.
This is surprising, since much less effort was placed in tuning RNN-based models than in a production scenario.
We therefore conclude that, in contrast to prior literature, using RNNs and RNN-modulated renewal processes generally improve intermittent demand forecasts, and should be considered as a promising tool in IDF.

Finally, we test our idea of using TPPs directly for IDF.
The UCI data set is the only one where exact purchase records with timestamps are available.
In Table~\ref{tab:results_ctrp}, we compare Static and RNN-based models with their continuous time counterparts.
Specifically, we train continuous-time models on interdemand times of individual purchase orders instead of the aggregated intermittent demand series.
We sample forward from these models, and aggregate the predicted purchase events in the last 6 time periods.
We compare these with the forecast accuracy of samples taken from discrete-time models.
We can report only a slight improvement of predictive performance, which is not consistent across different metrics.
Although we find that this approach can bring improvements over discrete-time IDF methods, we believe further research is needed before conclusions can be drawn.

\begin{table}[!ht]
\centering
\caption{Experiment results for continuous-time models}
\label{tab:results_ctrp}
\begin{tabular}{@{}lcccccccccc@{}}
\toprule
{\bf Model}            & \textbf{P50 Loss} & \textbf{P90 Loss} & \textbf{MAPE} & \textbf{sMAPE} & \textbf{RMSE} & \textbf{RMSSE} \\  \midrule
Static G-Po    \cCro    &  2.059      &  3.043                 &  0.714       &  1.739      &   7.488       & 3.514       \\
Static G-NB    \cSny    &  2.054      &  3.041                 &  0.655       &  1.729      &   7.458       & {\bf 3.467 }\\
RNN G-Po                &  2.071      &  2.906                 &  0.676       &  1.699      &   7.393       & 3.562       \\
RNN G-NB                &  {\bf 2.039}&  {\bf 2.870}           &  0.669       &  1.686      &   7.368       & 3.566       \\
Static E-Po (cont.)     &  2.062      &  3.039                 &  0.715       &  1.739      &   7.489       & 3.518       \\
Static E-NB (cont.)     &  2.055      &  3.015                 &  0.678       &  1.726      &   7.449       & 3.477       \\
RNN E-Po (cont.)        &  2.055      &  2.937                 &  {\bf 0.652} &  1.708      &   7.416       & 3.544       \\
RNN E-NB (cont.)        &  2.059      &  2.871                 &  0.754       &  {\bf 1.677}&   {\bf 7.359 }& 3.738       \\\bottomrule
\end{tabular}
\end{table}

\section*{Discussion and Conclusion}\label{sec:discussion}

IDF is a uniquely challenging problem. 
The definition of good forecasts is as elusive as the techniques used to produce them.
Most previous works in IDF have focused on point forecast methods with strict assumptions, starting from that of Croston \cCro{}.
When models were proposed, these were rather limited in the statistical patterns that they could accomodate.

Connecting IDF models to discrete-time renewal processes yields a flexible framework for building stochastic models of intermittent time series.
By extending simple, widely used models with flexible interdemand time distributions, we are able to recover various common demand arrival patterns such as aging, periodicity, and clustering.
We believe that, especially in the context of probabilistic forecasting, these models are not only useful tools as presented here, but also open up promising avenues for further research.

RNNs, used as a subcomponent in this framework, lead to substantial increases in forecast accuracy.
We demonstrate, on both synthetic and real data sets, that common patterns in IDF are much better represented by RNNs than existing methods.
Moreover, our results on neural networks can be seen as a proof of concept rather than a full treatment of possible deep neural network architectures for IDF.
This is since no hyperparameter optimization was performed when learning models involving RNNs.
Indeed, different model architectures, learning rates, regularization schemes will contribute further to the favorable results we report here.

Our study also highlights the intimate connection between temporal point processes and intermittent demand models.
Neural temporal point process models, a recent research direction in machine learning, appear as continuous-time instances of our framework.
These models are especially well-suited to scenarios in which purchase order data are directly available, removing the need to decide on how to temporally aggregate demand instances, and directly taking advantage of the full temporal granularity.
While our empirical results on the limited data available do not yield substantive evidence for taking this approach, we strongly believe further effort should be invested in the study of TPPs in forecasting.

\appendix

\section*{Supporting information}

\paragraph{S1 Appendix: Negative Binomial Distribution.}  \label{appx:neg_bin}

We use a slightly altered version of the negative binomial distribution.
Recall that the (generalized) negative binomial random variable is defined with the probability mass function
\[
\IP{X = k} = {k + r - 1 \choose k} (1 - \pi)^{r} \pi^k,
\]
where $k \in \{0, 1, 2, \dots\}, r > 0, \pi \in [0, 1]$. 
We ``shift'' the distribution for consistency with definitions of interarrival times and demand sizes. That is, we define $Y \in \{1, 2, \cdots\}$ such that $Y = X + 1$.
We then have,
\[
\IP{Y = k} = {k + r - 2 \choose k - 1} (1 - \pi)^{r} \pi^{(k-1)}.
\]
We also parameterize the distribution with the ``mean-dispersion'' convention, defining
\[
\mu = \dfrac{\pi r}{1 - \pi} + 1  \qquad \nu = \dfrac{1}{1-\pi},
\]
with $\mu > 1, \nu > 1$. 
Note $\ex{Y} = \mu$ and that $\varn{Y}/(\ex{Y}-1) = \varn{X}/\ex{X} = \nu$ is the ``shifted'' variance-to-mean ratio. 
Concretely, when we define $Y \sim \mathcal{NB}(\mu, \nu)$, we refer to the random variable determined by the probability mass function
\[
p(k) = \IP{Y = k} = {k + \frac{\mu-1}{\nu-1} - 2 \choose k - 1} 
\left(\dfrac{1}{\nu}\right)^{\frac{\mu-1}{\nu-1}} 
\left(1 - \dfrac{1}{\nu}\right)^{(k-1)}.
\]
The distribution function of $Y$ is
\[
F_Y(k) = \IP{Y \le k} = 1 - I_{1 - 1/\nu}\left(k, \dfrac{\mu-1}{\nu-1}\right),
\]
where $I_x(a, b)$ is the regularized {\em incomplete Beta function}.

Finally, let us give an explicit form for the {\em hazard rate} implied by negative binomial random variables,
\begin{align}\label{eq:nb_hazard_rate}
h(k) &= \dfrac{p(k)}{1 - F(k-1)} = \dfrac{F(k) - F(k-1)}{1 - F(k-1)} 
= \dfrac{I_{1 - 1/\nu}(k-1, r) - I_{1 - 1/\nu}(k, r)}{I_{1 - 1/\nu}(k-1, r)} \\
&= 1 - \dfrac{I_{1 - 1/\nu}(k, r)}{I_{1 - 1/\nu}(k-1, r)},    
\end{align}
where we keep $r = \dfrac{\mu - 1}{\nu - 1}$ and take $F(0) = 0$.

\paragraph*{S2 Appendix: Forecast Functions for Static NB Models.}  \label{appx:static_nb_forecasts}

Many of the models we use in this work do not admit closed forms for forecast functions, \ie, efficient closed-form estimators for the conditional mean forecast.
Notable exceptions are Static models with with negative binomial interdemand times. 
Indeed, an analytical expression for $\ex{Y_{n+1} | \HH_n}$ is possible, where use $\HH_n$ to denote the filtration up to demand review period $n$.
Let $Z_n = \indf{Y_n > 0}$ is a binary random variable that is 1 when period $n$ has positive demand.
We denote the number of demand points in $\HH_n$ as $i$, \ie, $i = \sum_{\nu=0}^n Z_\nu$. 
Then,
\[
\ex{Y_{n+1} | \HH_n} = \ex{Z_{n+1} M_{i+1} | \HH_n} = \ex{Z_{n+1} | \HH_n} \ex{M_{i+1} | \HH_n} = \ex{Z_{n+1} | \HH_n} \ex{M_{i+1}}.
\]
Here, the first equality is by definition.
The second follows from our model assumption that demand sizes and intervals are independent. 
Finally, the third equality follows from our assumption that demand sizes are independent of the past.
Moreover,
\[
\ex{Z_{n+1} | \HH_n} = \IP{Z_{n+1} = 1 | \HH_n} = \IP{Z_{n+1} = 1 | T_i},
\]
where we define $T_i = \sum_{j=0}^i Q_j$ as the {\em time} of the previous nonzero demand.
This follows from our renewal process assumption.
That is, given the time of the previous issue point, the time of the next point is conditionally independent of the history.
We can rewrite 
\[
\IP{Z_{n+1} = 1 | T_i} = \IP{Q_{i+1} = n - T_i + 1 | Q_{i+1} > n - T_i} = h_{Q}(n - T_i)
\]
where $h_{Q}$ is the {\em hazard rate} (\ref{eq:nb_hazard_rate}). Finally, we have
\[
\ex{Y_{n+1} | \HH_n} = h_{Q}(n - T_i) \ex{M_{i+1}}.
\]
Conditioning on the EWMA process, similar expressions can be derived easily for EWMA-type self-modulating DTRPs. 
Finally, tools from renewal theory can be used to characterize multi-step forecasts, see \eg, discussions in \cite{feller_introduction_1957}.

\paragraph*{S3 Appendix: Convergence Problem.}  \label{appx:convergence}

In the Models section, we introduced self-modulated DTRP models, and commented that these models suffered from a similar ``convergence'' issue as in previous {\em nonnegative EWMA} models \cite{grunwald_properties_1997,shenstone_stochastic_2005}.

In our construction, we defined the size-interval sequence $\{(M_i, Q_i)\}$ on positive integers $\{1, 2, \cdots\}$.
Therefore, in contrast to the ``convergence to zero'' issue outlined in, \eg, \cite{shenstone_stochastic_2005}, our models are plagued by convergence to dense trajectories of demand sizes 1. 
This can be a slightly more ``desirable'' problem, although we should still caution that these models are not suited for forecasts with long lead times.

More formally, let 
\begin{subequations}\label{eq:appxc_m1}
\begin{align}
M_i - 1 &\sim \mathcal{PO}(\hat{M}_{i-1} - 1), \label{eq:appx_ewma_1}  \\
\hat{M}_{i} &= (1 - \beta) \hat{M}_{i-1} + \beta M_i, \label{eq:appx_ewma_2} 
\end{align}
\end{subequations}
where $0 < \beta \le 1$. 
Below, we give a statement of the ``convergence to one'' issue with an argument that follows \cite{grunwald_properties_1997}, albeit with a slightly more accessible proof. 
\begin{proposition}
Let $M_i$ be defined as in (\ref{eq:appxc_m1}). $M_i \longrightarrow 1$ as $i \uparrow \infty$ almost surely.
\end{proposition}
\begin{proof}
First note that $\{\hat{M}_i\}$ is a positive martingale, since  $\exs{\hat{M}_i | M_{1:i-1}} = \exs{\hat{M}_i | \hat{M}_{i-1}} = \hat{M}_{i-1}$. 
By the martingale convergence theorem \cite[Sec 7.4]{shiryaev}, $\hat{M}_i \rightarrow X$ a.s. for some random variable $X$, and $X \in [1, \infty)$ naturally. 
By (\ref{eq:appx_ewma_2}), it is also clear that $M_i \rightarrow X$ a.s.

However, we must then note $M_i - \hat{M}_{i-1} = M_i - \exs{M_i} \rightarrow 0$.
In other words, $X$ is a degenerate random variable, taking a value in $[1, \infty)$ with probability 1.
However, $M_i$ is degenerate iff $\varn{M_i} = 0$. 
Noting that $M_i$ is defined as a shifted Poisson random variable, and $\varn{M_i} \rightarrow \varn{X} = X - 1 = 0$, we have the desired proof.
\end{proof}
Similarly, this proof can be extended to geometric and negative binomial random variables to show that
\begin{corollary}
$Q_i \longrightarrow 1$ as $i \uparrow \infty$ a.s.
\end{corollary}

{\em Static} models do not suffer from the convergence problem. 
Moreover, \cite{snyder_forecasting_2012} discuss a set of models with stationary mean processes that mitigate this issue. 
This is done by defining both sizes and intervals as a stationary autoregressive process.
For example,
\[
\hat{M}_{i} = (1 - \varphi - \beta) \mu + \beta \hat{M}_{i-1} + \varphi M_i,
\] 
where $\varphi + \beta < 1$, $\varphi, \beta, \mu \in \IR_{+}$, and $Q_i$ analogously.

\paragraph*{S4 Appendix: Forecast accuracy metrics.}
Our first two forecast accuracy metrics are based on forecast distributions, while the rest are some of the most commonly used metrics for measuring point forecast accuracy.
\begin{itemize}
\item {\bf P50} and {\bf P90 Loss} are metrics based on the quantile loss. 
Letting $\hat{y}(\rho)$ denote the quantile estimate---here obtained via sampling from the forecast distribution---we denote
\[
\text{QL}_\rho(\obs_n, \widehat{\obs}_n(\rho)) 
=\begin{cases} 
2\cdot\rho\cdot(\obs_n - \widehat{\obs}_n(\rho)), & \obs_n - \widehat{\obs}_n(\rho) > 0, \\
2\cdot(1-\rho)\cdot(\widehat{\obs}_n(\rho) - \obs_n), & \obs_n - \widehat{\obs}_n(\rho) \leqslant 0.
\end{cases}
\]
We define {\bf PXXLoss}, given lead time $L < N$,
\[
\text{PXXLoss}_\rho (\mathbf{y}, \mathbf{\hat{y}}) 
= \dfrac{1}{LM} \sum_{i=1}^M \sum_{n=1}^L \text{QL}_\rho\left( y_{in}, \hat{y}_{in} \right)
\]
We report {\bf P50Loss} and {\bf P90Loss}, setting $\rho = 0.5$ and $\rho = 0.9$ respectively.
\item {\bf MAPE} is symmetric mean absolute percentage error. As this metric is undefined when $y_{in} = 0$, we discard these instances and consider only time steps when $y_{in} > 0$.
\[
\text{MAPE}(\mathbf{y}, \mathbf{\hat{y}})  
= \dfrac{1}{M} \sum_{i=1}^M 
\left(\dfrac{1}{\sum_n [y_{in} > 0]} \sum_{n=1}^L 
[y_{in} > 0] \dfrac{|y_{in} - \hat{y}_{in}|}{|y_{in}|} \right).
\]
\item {\bf sMAPE} is symmetric mean absolute percentage error. 
\[
\text{sMAPE}(\mathbf{y}, \mathbf{\hat{y}})  = \dfrac{2}{LM} \sum_{i=1}^M \sum_{n=1}^L 
\dfrac{|y_{in} - \hat{y}_{in}|}{|y_{in}| + |\hat{y}_{in}|}.
\]
\item {\bf RMSE} is root mean squared error.
\[
\text{RMSE}(\mathbf{y}, \mathbf{\hat{y}}) = \left(\dfrac{1}{LM} \sum_{i=1}^M \sum_{n=1}^L (y_{in} - \hat{y}_{in})^2\right)^{\frac{1}{2}}.
\]
\item Finally, {\bf RMSSE} is root mean squared scaled error, also used in the M5 forecasting competition \cite{m5guide}.
\[
\text{RMSSE}(\mathbf{y}, \mathbf{\hat{y}}) = \dfrac{1}{ML} \sum_{i=1}^M
\sum_{n=1}^L\dfrac{(y_{in} - \hat{y}_{in})^2}{\frac{1}{L'} \sum_{n'=1}^{L'-1} |y_{i,n'+1} - y_{in'}|}.
\]
Here, we let $n'$ index the in-sample time series (before the forecast horizon), that has length $L'$.
\end{itemize}

\section*{Acknowledgments}
The authors would like to thank Nikolaos Kourentzes for kindly providing the Auto and RAF data sets.

\bibliography{main}

\end{document}